\documentclass[letterpaper]{article} 
\usepackage{aaai24}  
\usepackage{times}  
\usepackage{helvet}  
\usepackage{courier}  
\usepackage[hyphens]{url}  
\usepackage{graphicx} 
\usepackage{natbib}  
\usepackage{caption} 
\usepackage{algorithm}
\usepackage{algorithmic}

\usepackage{newfloat}
\usepackage{listings}
\DeclareCaptionStyle{ruled}{labelfont=normalfont,labelsep=colon,strut=off} 
\floatstyle{ruled}
\newfloat{listing}{tb}{lst}{}
\floatname{listing}{Listing}
\usepackage{amsmath}
\usepackage{amssymb}
\usepackage{amsthm}
\DeclareMathOperator*{\argmax}{arg\,max}
\usepackage{graphicx}
\usepackage{mathtools}
\DeclarePairedDelimiter{\ceil}{\lceil}{\rceil}
\usepackage{xr}

\usepackage{siunitx}
\usepackage{adjustbox}
\usepackage[utf8]{inputenc}
\usepackage[T1]{fontenc}
\theoremstyle{plain}
\newtheorem{theorem}{Theorem}[section]

\newtheorem{lemma}[theorem]{Lemma}

\theoremstyle{definition}

\newtheorem{assumption}[theorem]{Assumption}
\theoremstyle{remark}

\usepackage{textcomp}
\usepackage{soul}
\usepackage{tablefootnote}

\title{A Primal-Dual Algorithm for Hybrid Federated Learning}
\author {
    Tom Overman\textsuperscript{\rm 1},
    Garrett Blum\textsuperscript{\rm 2},
    Diego Klabjan\textsuperscript{\rm 3}
}
\affiliations {
    \textsuperscript{\rm 1}Department of Engineering Sciences and Applied Mathematics, Northwestern University\\
    \textsuperscript{\rm 2}Department of Mechanical Engineering, Northwestern University\\
    \textsuperscript{\rm 3}Department of Industrial Engineering and Management Sciences, Northwestern University\\
    tomoverman2025@u.northwestern.edu, garrettblum2024@u.northwestern.edu, d-klabjan@northwestern.edu
}

\begin{document}

\maketitle

\begin{abstract}
  Very few methods for hybrid federated learning, where clients only hold subsets of both features and samples, exist. Yet, this scenario is extremely important in practical settings. We provide a fast, robust algorithm for hybrid federated learning that hinges on Fenchel Duality. We prove the convergence of the algorithm to the same solution as if the model is trained centrally in a variety of practical regimes. Furthermore, we provide experimental results that demonstrate the performance improvements of the algorithm over a commonly used method in federated learning, FedAvg, and an existing hybrid FL algorithm, HyFEM. We also provide privacy considerations and necessary steps to protect client data.
\end{abstract}

\section{Introduction}
Federated learning (FL) has quickly become a top choice for privacy-aware machine learning \citep{fl}. The basic premise of federated learning is that a group of external nodes called clients hold parts of the data and a central server coordinates the training of a model representative of these data but without directly accessing the data itself. This requires the clients to train local models, then pass some information (such as model weights) to the server where the server aggregates the clients' contributions to update its global model. The goal of FL is to build algorithms that result in convergence to a similar objective value as the centralized case, as if the server had access to all data directly, and perform well over various problem settings with minimal communication overhead.

Federated learning can be classified based on how the data are gathered on the clients. In horizontal FL, each client holds a subset of the samples that contain all of their features. In vertical FL, each client holds all of the samples but only a subset of each sample's features. These are both special cases of hybrid FL where each client contains a subset of the samples and a subset of the features. 

Hybrid FL is less studied than the case of horizontal and vertical FL, but it is still extremely important in practice. An example of hybrid FL is the case where multiple hospitals wish to build a central model but cannot directly share data between hospitals due to privacy laws. Each hospital has a subset of all of the patients, and since each patient may have visited multiple hospitals, the patient's features are split between many hospitals. The same situation exists in banking for fraud detection with explainable convex models \cite{lv}.

Another example is in telecommunication where each tower collects data from cell devices that ping the tower. Each cell tower has different specifications and thus collects different measurements than other towers. Therefore, each tower collects different features, and since not every user connects to every tower and most users interact with multiple different towers, the samples are also split across towers. 

We introduce a primal-dual algorithm, Hybrid Federated Dual Coordinate Ascent (HyFDCA), that solves convex problems in the hybrid FL setting. This algorithm extends CoCoA, a primal-dual distributed optimization algorithm introduced by \citet{cocoa1} and \citet{cocoa2}, to the case where both samples and features are partitioned across clients. We provide privacy considerations that ensure that client data cannot be reconstructed by the server. Next, we provide proofs of convergence under various problem settings including special cases where only subsets of clients are available for participation in each iteration. The algorithm and associated proofs can also be utilized in the distributed optimization setting where both samples and features are distributed. As far as we know, this is the only algorithm in the doubly distributed case that has guaranteed convergence outside of block-splitting ADMM developed by \cite{block_splitting_admm}. ADMM has not been designed with FL in mind, but the algorithm has no data sharing. On the down side, block-splitting ADMM requires full client participation which makes it much more restrictive than HyFDCA and essentially impractical for FL. HyFDCA is also the only known hybrid FL algorithm that converges to the same solution as if the model is trained centrally. Finally, we provide extensive experimental results that demonstrate the performance improvements of HyFDCA over FedAvg, a commonly-used FL algorithm \citep{fedavg}, and HyFEM, a hybrid FL algorithm \citep{hyfem}.

Our main contributions in this work are as follows:
\begin{enumerate}
    \item Provide HyFDCA, a provably convergent primal-dual algorithm for hybrid FL. The proofs cover a variety of FL problem settings such as incomplete client participation. Furthermore, the convergence rates provided for the special cases of horizontal and vertical FL match or exceed the rates of popular FL algorithms designed for those particular settings.
    \item Provide the privacy steps that ensure privacy of client data in the primal-dual setting. These principles apply to future efforts in developing primal-dual algorithms for FL.
    \item Demonstrate that HyFDCA empirically outperforms both FedAvg and HyFEM in the loss function value and validation accuracy across a multitude of problem settings and datasets. We also introduce a hyperparameter selection framework for FL with competing metrics using ideas from multiobjective optimization.
\end{enumerate}

In Section \ref{sec:related_work}, we discuss work that has been done in the vertical and horizontal settings and the lack of algorithms that exist for the hybrid setting. We then highlight the improvements that HyFDCA provides in theory and practice. In Section \ref{sec:algo}, we introduce HyFDCA and privacy considerations that protect client data. In Section \ref{sec:analysis}, we analyze convergence of HyFDCA and provide convergence results in a variety of practical FL problem settings. In Section \ref{sec:experiments}, we present experimental results on three separate data sets and compare the performance of HyFDCA with FedAvg and HyFEM.

\section{Related Work}
\label{sec:related_work}
There has been significant work in developing primal-dual algorithms using Fenchel Duality for distributed optimization where samples are distributed. One of the leading frameworks on this front is CoCoA. However, these algorithms do not properly handle data that is distributed over both samples and features. This extension from partitioning data over a single axis direction to both directions is not trivial, especially in the primal-dual case where now multiple clients share different copies of the same dual variables and primal weights. D3CA is the first algorithm to extend CoCoA to the case where data are distributed over samples and features \citep{d3ca}. However, D3CA has no convergence analysis and has convergence problems in practice with small regularization constant. HyFDCA fixes these issues with D3CA and is altered to ensure that the privacy requirements for FL are met. Block-splitting ADMM is the only other algorithm that can handle distributed samples and features. However, as \cite{d3ca} show, the empirical performance of block-splitting ADMM is poor and full client participation is needed. HyFDCA and the associated proofs, while focused on the federated setting, can also be utilized in the distributed optimization setting where both samples and features are distributed.

There has been substantial work in horizontal FL where samples are distributed across clients but each sample contains the full set of features. One of the most commonly used algorithms is FedAvg which, in essence, computes model weights on each client using stochastic gradient descent (SGD), then averages together these model weights in an iterative fashion. FedAvg can be naively extended to the hybrid FL case by computing client weights locally, as before, then concatenating the model weights and averaging at the overlaps. From now on, this modified version of FedAvg is what is meant when referencing FedAvg in the hybrid FL setting. Empirical results in Section \ref{sec:experiments} demonstrate that this naive extension is not satisfactory, and focused algorithms with satisfactory performance specifically for hybrid FL must be developed. The convergence rate of HyFDCA matches FedAvg \citep{fedavg_convergence} in the special case of horizontal FL. Furthermore, HyFDCA does not require smooth loss functions unlike FedAvg, making the convergence results more flexible.

FedDCD is an approach for using dual methods for FL, but is limited to the regime of horizontal FL \citep{fedDCD}. The extension to hybrid FL is not clear as now multiple clients hold copies of the same dual variables and the local coordinate descent that FedDCD performs is no longer valid. Furthermore, the proof results given for FedDCD require smooth loss functions which eliminate many common loss functions such as hinge loss. Our convergence results do not require smoothness of the loss functions. FedDCD does not mention privacy considerations in the case that the mapping between primal and dual variables can be inverted to reveal information about the data held on clients. We address these privacy concerns with suitable homomorphic encryption steps in HyFDCA. We note that to the best of our knowledge, HyFDCA is the only primal-dual algorithm that can handle vertical FL.

There has been substantially less work in vertical FL where each client has all of the samples but only a subset of the features. Some approaches exist such as FedSGD (vertical variant) and FedBCD which rely on communicating relevant information between clients to compute stochastic gradients despite only holding a portion of the features \citep{fedbcd}. However, these algorithms do not work in the hybrid FL case we are exploring. Furthermore, they require communication of this gradient information between clients instead of just passing information through the server. In addition, the convergence rate of HyFDCA in the special case of vertical FL is faster than FedBCD, demonstrating that while HyFDCA is designed to handle hybrid FL, it also enjoys improvements over existing methods in the special cases of horizontal and vertical FL.

To the best of our knowledge, there is only one other algorithm that focuses on hybrid FL, HyFEM. This algorithm uses a feature matching formulation that balances clients building accurate local models and the server learning an accurate global model. This requires a matching regularizer constant that must be tuned based on user goals and results in disparate local and global models. Furthermore, the convergence results provided for HyFEM only claim convergence of the matching formulation not of the original global problem and require complete client participation. In other words, though HyFEM can converge to a solution, it may be substantially worse or simply divergent with incomplete client participation than if the same data are used to train a model centrally as we show in Section \ref{results_section}. This work is substantially different than our approach which uses data on local clients to build a global model that converges to the same solution as if the model is trained centrally. Properly tuning the matching constant takes significant computational resources which is not a problem for HyFDCA. Furthermore, due to HyFEM not solving the original optimization problem, we show that the empirical results for convex problems are significantly worse than HyFDCA.

\section{The Primal-Dual Algorithm}
\label{sec:algo}
The goal is to solve the following minimization problem that consists of a strongly convex, L2 regularizer and a sum of convex loss functions
\begin{equation}
    \min_{w\in \mathbb{R}^M} P(w) = \frac{\lambda}{2} ||w||^2 + \frac{1}{N}\sum_{i=1}^N l_i(w^Tx_i)
\end{equation}
where $w$ are the weights of the model, $M$ is the total number of features, $N$ is the total number of samples, $\lambda$ is the regularization parameter that influences the relative importance of regularization, $x_i$ is the $i$-th sample, and $l_i$ are sample specific loss functions. This class of problems encompasses many important models in machine learning including logistic regression and support vector machines (SVM).

Our approach takes advantage of the Fenchel Dual of this problem, which is defined as

\begin{equation}
    \max_{\alpha \in \mathbb{R}^N} D(\alpha) = -\frac{\lambda}{2} || \frac{1}{\lambda N} \sum_{i=1}^N \alpha_i x_i||^2 - \frac{1}{N} \sum_{i=1}^N l_i^*(-\alpha_i)
\end{equation}
where $\alpha$ are the dual variables and $l^*$ is the convex conjugate of $l$. There is a convenient relationship between optimal primal, $w^*$, and dual variables, $\alpha^*$, defined by $w^*=\frac{1}{\lambda N}\sum_{i=1}^N\alpha^*_ix_i$. When $l_i$ are convex, we have that $P(w^*)=D(\alpha^*)$.

\subsection{The HyFDCA Algorithm}

The main idea of HyFDCA, shown in Algorithm \ref{mainalgo}, is that each client performs a local dual coordinate ascent to find updates to the dual variables. This local method utilizes the inner product of the primal weights and the data, and thus a secure way of finding this inner product across clients that contain sections of each sample is provided in Algorithm \ref{secureinnerproduct}. The details of the local dual method are shown in Algorithm \ref{localdualmethod}. These dual updates from clients are averaged together and then used to update the global dual variables held on the server. These updated dual variables are then sent back to the clients where they each compute their local contribution of the primal weights. These are then sent back to the server and aggregated. The steps to compute these global primal weights are shown in Algorithm \ref{primalaggregation}. A diagram of HyFDCA that demonstrates each step is shown in Figure \ref{flowchart_diagram}.

We introduce some additional notation. Set $\mathcal{B}_n$ is the set of clients that contain sample $n$; $\mathcal{I}_k$ is the set of samples available on client $k$; $\mathcal{M}_k$ is the set of features available to client $k$; and $\mathcal{K}_m$ is the set of clients that contain feature $m$. Furthermore, $x_{k,i}$ is the subset of sample $i$ available to client $k$, and $x_{k,i,m}$ is the value of feature $m$ of sample $i$ located on client $k$.

\begin{algorithm}[ht!]
 \caption{HyFDCA}
\label{mainalgo}
\begin{algorithmic}
\STATE Initialize $\alpha^0=0, w^0=0$, and $\hat{w}_{0}=0$.
\STATE Set $ip_{k,i}=0$ for every client $k$ and $i \in \mathcal{I}_k$.
 \FOR{t=1,2,...T}
    \STATE Given $\mathcal{K}^t$, the subset of clients available in the given iteration 
    \STATE Find $\overline{\mathcal{K}^t} = \{k: k \in \mathcal{K}^t \text{ and } k \notin \mathcal{K}^{t-1}\}$
    \STATE Send $\text{enc}(\alpha_{0}^t)$ to clients $k \in \overline{\mathcal{K}^t}$
  \STATE PrimalAggregation($\overline{\mathcal{K}^t}$)
  \STATE SecureInnerProduct($\mathcal{K}^t$)
  
  \FOR{\textbf{all clients} $k \in \mathcal{K}^t$}
    \STATE $\Delta \alpha_k^t$=LocalDualMethod($\alpha_k^{t-1},w_k^{t-1},x_i^Tw_0^{t-1}$)
    \STATE Send all $\text{enc}(\frac{\gamma_t}{|\mathcal{B}_n|}\Delta \alpha_{k}^t)$ to server
  \ENDFOR
  \FOR{n=1,2,...,N}
    \STATE $\text{enc}(\alpha_{0,n}^t) = \text{enc}(\alpha_{0,n}^{t-1}) + \sum_{b \in \mathcal{B}_n}\text{enc}(\frac{\gamma_t}{|\mathcal{B}_n|}\Delta\alpha_{b,n}^t)$
  \ENDFOR
  \STATE Send $\text{enc}(\alpha_{0}^t)$ to clients $k \in \mathcal{K}^t$ and clients decrypt
  \STATE PrimalAggregation($\mathcal{K}^t$)
  \STATE SecureInnerProduct($\mathcal{K}^t$)
 \ENDFOR
 \end{algorithmic}
\end{algorithm}

Due to the mapping between primal and dual variables, $w=\frac{1}{\lambda N}\sum_{i=1}^N\alpha_ix_i =\textbf{A}\alpha$, care needs to be taken to prevent the reconstruction of $\textbf{A}$ from iterates of $w$ and $\alpha$. The server could collect $w^t$ and $\alpha^t$ for many $t$ and construct a system of linear equations $\mathcal{W}=\textbf{A}\mathcal{A}$ where $\mathcal{W}$ collects iterates of $w$ in its columns and $\mathcal{A}$ collects iterates of $\alpha$ in its columns. This would allow for the solution of $\textbf{A}$ or the approximate solution using least-squares if $\textbf{A}$ is not square. For this reason, either $\alpha$ or $w$ should be encrypted to prevent this reconstruction of data. Because $w$ is used by the central model for inference on new data, we choose to encrypt $\alpha$ using homomorphic encryption. All aggregations are done on the server, and the server knows the sample IDs of each client through private set intersection \cite{9343209}.

Homomorphic encryption is a technique for encrypting data and preserving certain arithmetic operations in the encrypted form \citep{homomorphic_encryption}. For example, in additive homomorphic encryption the following holds: $\text{enc}(X) + \text{enc}(Y) = \text{enc}(X + Y)$. There are numerous algorithms for homomorphic encryption and new, faster algorithms are invented frequently. For example, the Paillier cryptosystem takes on average 18.882 ms for encryption, 18.865 ms for decryption, and 0.054 ms for addition in the encrypted state \citep{homo_enc_speed}. HyFDCA uses homomorphic encryption in several steps to ensure that the server can perform aggregation operations but not reconstruct the underlying data that belongs to the clients.

In addition, the communication of the inner product information poses a similar problem. If we define $b_i^t = (w^t)^Tx_i$, then the server could collect iterates of $b$ and $w$ and form a system of linear equations $B=\mathcal{W}x_i$ where $\mathcal{W}$ collects $w$ in its rows and $B$ is a column vector of the corresponding $b_i$. This system could then be solved for $x_i$. For this reason, the inner product components passed to the server from the clients must be encrypted using additive homomorphic encryption. So far, we have addressed the server reconstructing data, however, another concern is the clients themselves reconstructing data from other clients. It is important that the clients are only sent the dual variables corresponding to the samples on that client and only the primal weights corresponding to the features on that client. With this information they would only be able to reconstruct their local data.

The dual and primal variables are sent between the server and clients at the beginning and end of each iteration to ensure that clients are not using stale information because they may not participate in every iteration. Several round-trip communications (RTC) are necessary for the difficult case of hybrid FL with incomplete client participation. For the easier cases of horizontal and vertical FL explored in previous works, HyFDCA can be simplified without changing the framework, resulting in less communications. For horizontal FL with complete client participation, SecureInnerProduct and the first instance of PrimalAggregation can be removed which reduces each iteration to two RTC. Furthermore, if we assume complete client participation, then the first instances of SecureInnerProduct and PrimalAggregation can be removed and each iteration only has three RTC even in the hybrid FL setting.

\begin{algorithm}
\caption{SecureInnerProduct}
\label{secureinnerproduct}
\begin{algorithmic}
\STATE Input: Set of available clients $\mathcal{K}$ \\
  \FOR{\textbf{all clients} $k \in \mathcal{K}$}
  \FOR{\textbf{all samples} $i \in \mathcal{I}_k$}
  \STATE Client $k$ computes local $x_{k,i}^{T}w_k^t$ and encrypts this scalar using additive homomorphic encryption resulting in $\text{enc}(x_{k,i}^{T}w_k^t)$.
  \STATE $ip_{k,i}=\text{enc}(x_{k,i}^Tw_k^t)$.
  \STATE Send $ip_{k,i}$ to server.
  \ENDFOR
  \ENDFOR
  \FOR{\textbf{all samples} $i =1,2,...,N$}
    \STATE Server computes $\text{enc}(x_i^Tw^t)=\sum_{k\in \mathcal{B}_i}ip_{k,i}$. 
    \STATE Send to all clients $k \in \mathcal{K}$ all values $\text{enc}(x_i^Tw^t)$ for $i \in \mathcal{I}_k$.
  \ENDFOR
 \STATE Clients decrypt $\text{enc}(x_i^Tw_0^t)$ to obtain $x_i^Tw_0^t$.
 \end{algorithmic}
\end{algorithm}

\begin{algorithm}
 \caption{LocalDualMethod}
\label{localdualmethod}
\begin{algorithmic}
\STATE Input: $\alpha_{k}^{t-1}, w_k^{t-1}, x_i^Tw_0^t$
\STATE $D$ is a set of sample indices available to client $k$ of size $H$ randomly chosen without replacement
\STATE Let $\Delta \alpha_{k,i}^t=0$ for all $i\in\mathcal{I}_k$
 \FOR{$i \in D$}
 \STATE \text{Let $u_{i}^{t-1}\in\partial l_i(x_{i}^Tw_0^{t-1})$}
     \STATE $s_{k,i}=\argmax_{s\in [0,1]} \{ -l_i^*(-(\alpha_{k,i}^{t-1} + sc_k \gamma_t(u_{i}^{t-1}-\alpha_{k,i}^{t-1}))) -s c_k \gamma_t(w^{t-1})^Tx_i(u_{i}^{t-1}-\alpha_{k,i}^{t-1}) - \frac{\gamma_t^2}{2\lambda}(s c_k(u_{i}^{t-1}-\alpha_{k,i}^{t-1}))^2 \}$
     \STATE $\Delta \alpha_{k,i}^t=s_{k,i}c_k(u_{k,i}^{t-1}-\alpha_{k,i}^{t-1})$
 \ENDFOR
 \STATE Return $\Delta \alpha_k^t$.
\end{algorithmic}
\end{algorithm}

\begin{algorithm}
 \caption{PrimalAggregation}
\label{primalaggregation}
\begin{algorithmic}
\STATE Input: Set of available clients $\mathcal{K}$
    \FOR{\textbf{all clients} k $\in \mathcal{K}$}
    \FOR{\textbf{all features} $m \in \mathcal{M}_k$}
    \STATE $\hat{w}_{k,m} = \sum_{i \in I_k} \alpha^t_{k,i}  x_{k,i,m}$
    \ENDFOR
  \ENDFOR
  \STATE Update global $\hat{w}_{0,k,m}$ from available local $\hat{w}_{k,m}$
  \FOR{\textbf{all features} $m =1,2,...,M$}
    \STATE $w_{0,m}^{t} = \frac{1}{\lambda N}\sum_{k \in \mathcal{K}_m} \hat{w}_{0,k,m}$
  \ENDFOR
  \STATE Send $w_{0}^t$ to clients $k \in \overline{\mathcal{K}}$
  \end{algorithmic}
\end{algorithm}

\begin{figure}
\begin{center}
\includegraphics[width=.46\textwidth]{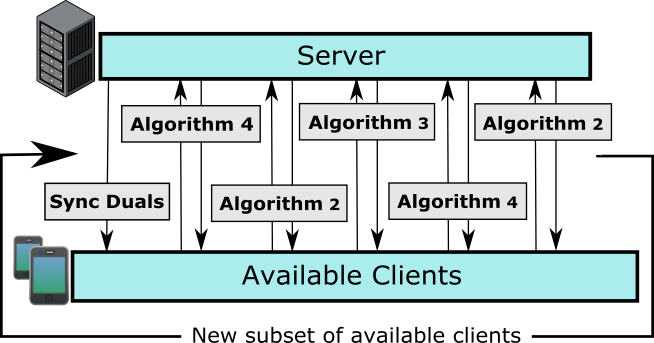}
\caption{Flowchart of HyFDCA. Each vertical arrow represents a communication of some information between clients and the server.}
\label{flowchart_diagram}
\end{center}
\end{figure}

\section{Convergence Analysis}
\label{sec:analysis}
We provide convergence proofs for HyFDCA in various problem settings. The proofs are in the Supplementary Materials of the extended paper \cite{overman2023primaldual}. We also note that no assumptions of IID data are made in these proofs, so they apply for non-IID settings.

\subsection{Hybrid Federated Setting with Complete Client Participation}
We first make the following assumptions of our problem setting.
\begin{assumption}
\label{assum1}
Loss functions $l_i\geq0$ are convex and $L$-Lipschitz functions. This is satisfied by many commonly-used loss functions in practice including logistic regression and hinge loss (support vector machines).
\end{assumption}
\begin{assumption}
The set of clients, $\mathcal{K}$, available at a given outer iteration is the full set of clients.
\end{assumption}
\begin{assumption}
\label{assum3}
The data are split among clients in the particular way shown in Figure \ref{clients_diagram}. The only assumption we make is that each sample on a particular client has the same features (the data are rectangular).
\begin{figure}[ht]
\begin{center}
\includegraphics[width=.38\textwidth]{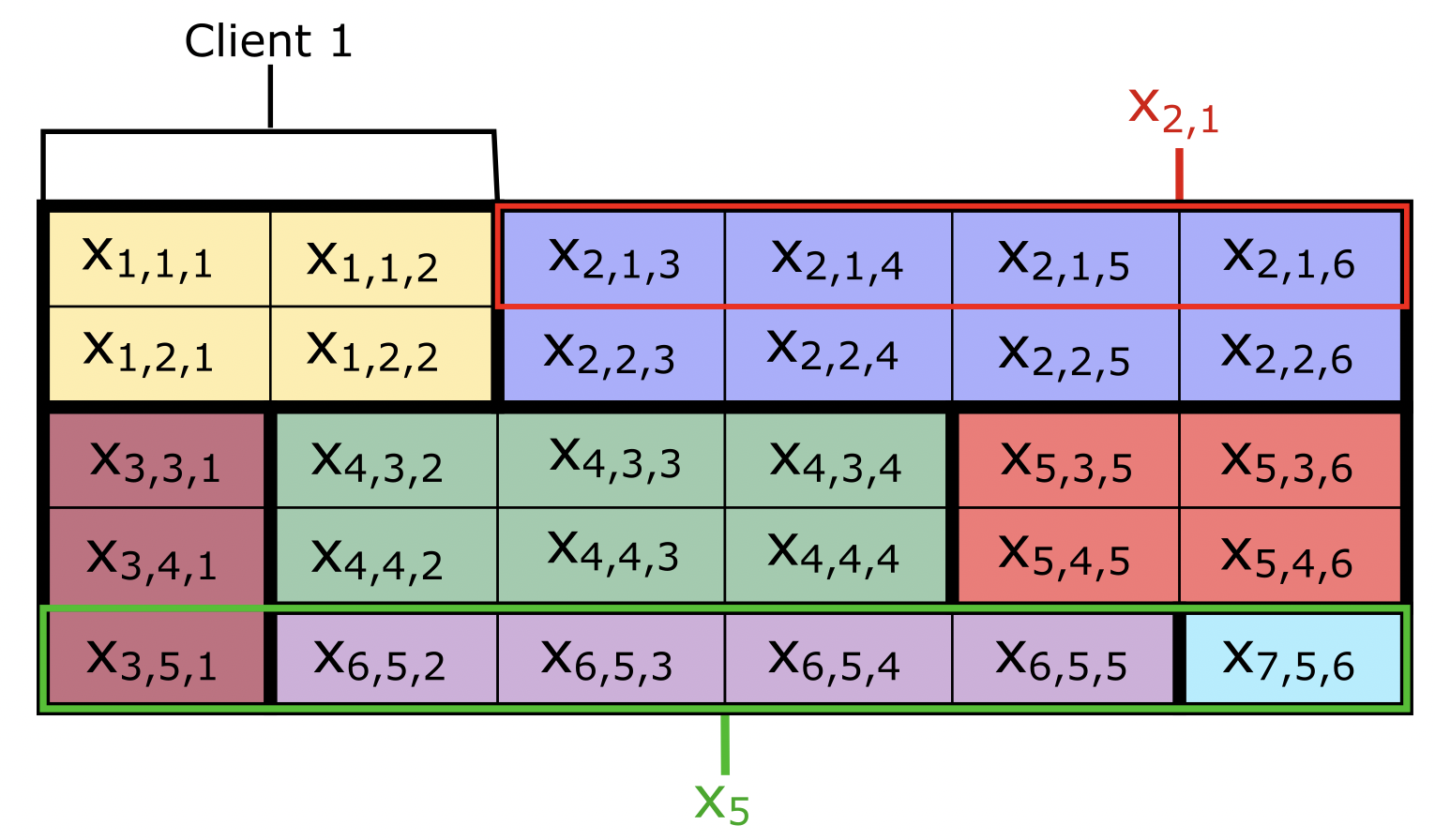}
\caption{Diagram of a possible data partitioning that follows Assumption \ref{assum3} for the convergence proof. Each client is shaded in a different color.}
\label{clients_diagram}
\end{center}
\end{figure}
\end{assumption}

\begin{theorem}
\label{thm1}
If Assumptions \ref{assum1}-\ref{assum3} are met, $\gamma_t=1$, and $c_k=N_k/N$ where $N_k$ is the number of samples on client $k$, then Algorithm \ref{mainalgo} results in the bound on the dual suboptimality gap, $\mathbb{E}[\varepsilon_D^t] \leq (1-\frac{s_t H}{N}) \mathbb{E}[\varepsilon_D^{t-1}] +\frac{s_t^2 H}{N}G$,
for any $s_t\in[0,1]$ and $G \leq \frac{2L^2}{\lambda}$, where $\varepsilon_D^t=D(\alpha^*)-D(\alpha^t)$.
\end{theorem}
Next, we find the bound on the suboptimality gap in terms of the number of iterations.
\begin{theorem}
\label{thm2_act}
If $\frac{H}{N}\leq1$, then for every $t\geq t_0$ we have
\begin{align}
 \mathbb{E}[\varepsilon_D^t] \leq \frac{2G}{1+\frac{H}{2N}(t-t_0)}
    \label{thm2}
\end{align}
where $t_0=\max\{ 0, \lceil \log(\mathbb{E}[\varepsilon_D^0]/G)\rceil \}$. This upper bound clearly tends to zero as $t\rightarrow \infty$.
\end{theorem}
The requirement of $H \leq N$ places an upper limit on the number of inner iterations that can be performed before aggregation across clients.

\subsection{Horizontal Federated Setting with Random Subsets of Available Clients}
The case of incomplete client participation for hybrid FL is difficult because of the presence of stale variables due to clients not participating in some iterations affecting several steps of the algorithm. More details on where these stale variables impact the algorithm are given in Section \ref{vertical_proof_sec}. The PrimalAggregation and SecureInnerProduct steps before and after the local updates alleviate these issues in practice, but problems still exist for convergence proofs. For this reason, we first approach the incomplete client participation case for horizontal FL where this problem does not exist.
\begin{assumption}
\label{assum_horiz2}
Data are split among clients such that every client has the full set of features but only a subset of samples (definition of horizontal FL). Furthermore, we assume that each client holds $N/K$ samples, where $K$ is the total number of clients.
\end{assumption}
\begin{assumption}
\label{assum_horiz3}
Each outer iteration, a random subset of clients is chosen to participate and perform updates. Each client has an equal probability of being chosen and the mean number of clients chosen for a given outer iteration is $P$. Thus the mean fraction of clients participating in a given outer iteration is $\frac{P}{K}$.
\end{assumption}

\begin{theorem}
\label{thm_horiz1}
If Assumptions \ref{assum1},\ref{assum_horiz2}, and \ref{assum_horiz3} are met, $\gamma_t=1$, $c_k=1$, and $\frac{PH}{N}\leq1$, then Algorithm \ref{mainalgo} results in the dual suboptimality gap for every $t\geq t_0$ as follows
\begin{align}
 \mathbb{E}[\varepsilon_D^t] \leq \frac{2G}{1+\frac{PH}{2N}(t-t_0)}
    \label{thm_horiz2}
\end{align}
where $t_0=\max\{ 0, \lceil \log(\mathbb{E}[\varepsilon_D^0]/G)\rceil \}$ and $G \leq \frac{2L^2}{\lambda}$. This upper bound clearly tends to zero as $t\rightarrow \infty$.
\end{theorem}
The requirement on $PH\leq N$ similarly places a limit on the amount of inner iterations that can be performed before aggregation. This result demonstrates that HyFDCA enjoys a convergence rate of $\mathcal{O}(\frac{1}{t})$ which matches the convergence rate of FedAvg. However, our convergence proof does not assume smooth loss functions whereas FedAvg does. This makes our convergence results more flexible in the horizontal FL setting.

\subsection{Vertical Federated Setting with Incomplete Client Participation}
\label{vertical_proof_sec}
We now explore the case of incomplete client participation for the vertical federated setting. We change the assumptions for how subsets of clients are available for participation in Assumption \ref{assum_vert3} because random client subsets impose some issues for vertical FL. If a particular $\alpha_i$ is updated, then $w_0$ needs to be updated using local data on each client. If one of these clients cannot provide its contribution to $w_0$, then these primal weights will be stale and thus $(w^{t-1}_0)^Tx_i$ used in LocalDualMethod will also be stale. We require a limit on the maximum number of iterations that a particular client can go without being updated. The reason this cannot be extended to the hybrid case is that if $w_0^{t-1}$ is updated, then a particular client will require $(w_0^{t-1})^Tx_i$ for any $i$, but now that samples are also split across clients, it may not be able to access all components of $(w_0^{t-1})^Tx_i$ if some of those clients are not available. This is unlike the vertical FL case where all samples belong to each client.
\begin{assumption}
\label{assum_vert2}
Data are split among clients such that every client has the full set of samples but only a subset of features (definition of vertical FL).
\end{assumption}
\begin{assumption}
\label{assum_vert3}
All of the clients are partitioned into $C\geq2$ sets where each subset of clients has $Q/C$ clients (we assume that $Q\mod C=0$). Let $\mathcal{B}_1, \mathcal{B}_2,...,\mathcal{B}_C$ be this partition. We then assume that client subsets are active (participating in a particular outer iteration) in the cyclic fashion. Thus the sequence of active clients is defined as $\mathcal{B}_1, \mathcal{B}_2,...,\mathcal{B}_C, \mathcal{B}_1,...,\mathcal{B}_C$,...
\end{assumption}

\begin{theorem}
\label{thm_vert1}
If Assumptions \ref{assum1}, \ref{assum_vert2}, and \ref{assum_vert3} are met, $\frac{H}{N}\leq 1$, $c_k=1$, and $\gamma_t=\frac{1}{t}$, then Algorithm 1 results in the following bound on the dual suboptimality gap for $t\geq C$
\begin{align*}
    \mathbb{E}[\varepsilon_D^{t}] \leq \frac{J_1 + J_2(\ln(t-C+1)+1)}{t^{H/N}}
\end{align*}
where $J_1 = C^{H/N}\mathbb{E}[\varepsilon_D^{C-1}]$, $J_2 = \frac{2HL^2(C+1)}{N\lambda}[(C-1)^4  +2(C-1)^2+1]$,
and $\mathbb{E}[\varepsilon_D^{C-1}]$ is bounded by a constant with the standard assumption that $l_i(0)\leq1$. This converges to zero as $t\rightarrow \infty$.
\end{theorem}
It is clear that for fastest convergence in an asymptotic sense we want $H/N$ to be large, however, this would also increase the magnitude of $J_1$ and $J_2$ which would in turn slow convergence in early iterations when $t$ is small and $J_1$ and $J_2$ dominate the bound. Furthermore, if we take $H/N=1$, then HyFDCA exhibits $\mathcal{O}(\frac{\log{t}}{t})$ convergence whereas FedBCD exhibits a slower $\mathcal{O}(\frac{1}{\sqrt{t}})$ convergence rate and requires full client participation. Thus, to the best of our knowledge HyFDCA exhibits the best convergence rates for vertical FL even with partial client participation.

We emphasize that Theorems \ref{thm2_act}-\ref{thm_vert1} demonstrate that HyFDCA in a particular federated setting converges to the same optimal solution as if all of the data are collected on a centralized device and trained with a convergent method. Furthermore, Theorems \ref{thm_horiz1} and \ref{thm_vert1} demonstrate that in the horizontal and vertical FL cases, HyFDCA is still guaranteed to converge to the optimal solution when only a subset of clients participate in each iteration. The convergence rates for the special cases of horizontal and vertical FL match or exceed the convergence rates of existing FL algorithms in those settings.

\section{Experimental Results}
\label{sec:experiments}
We investigate the performance of HyFDCA on several datasets and in several different problem settings (number of clients and percentage of available clients). These different problem settings cover the vast number of different environments seen in practice.

Three datasets are selected. MNIST is a database of handwritten digits where each sample is a $28$x$28$ pixel image \citep{MNIST}. News20 binary is a class-balanced two-class variant of the UCI ``20 newsgroup'' dataset, a text classification dataset \citep{news20_covtype}. Finally, Covtype binary is a binarized dataset for predicting forest cover type from cartographic variables \citep{news20_covtype}. 

We use the hinge loss function for $l_i$ in experiments. A practical variant of LocalDualMethod, shown in Algorithm \ref{practicalLocalDualMethod}, is used for experiments. Line 6 of Algorithm \ref{practicalLocalDualMethod} has a closed form solution of $\Delta \alpha_{k,i}^{t} = y_i(\max(0,\min(1,\lambda N(1-x_i^Tw_0^{t-1})+y_i\alpha_{k,i}^{t-1}))) - \alpha_{k,i}^{t-1}$, where $y_i$ is the class label for the $i$-th sample. Furthermore, the second occurrence of SecureInnerProduct in Algorithm \ref{mainalgo} is omitted for experiments because it does not improve empirical performance and incurred more communication cost.
\begin{algorithm}
\caption{LocalDualMethod (Practical Variant)}
\label{practicalLocalDualMethod}
\begin{algorithmic}
\STATE Input: $\alpha_{k}^{t-1}, w_k^{t-1}, x_i^Tw_0^t$
\STATE $D$ is a set of sample indices available to client $k$ of size $H$ randomly chosen without replacement
\STATE Let $\Delta \alpha_{k,i}^t=0$ for all $i\in\mathcal{I}_k$
 \FOR{$i \in D$}
     \STATE Find $\Delta \alpha_{k,i}^t$ that maximizes $-l_i^*(-(\alpha_{k,i}^{t-1}+\Delta \alpha_{k,i}^t)) - \frac{\lambda N}{2}(||w_0^{t-1}||^2 +\frac{2 \Delta \alpha_{k,i}^t}{\lambda N}(w_0^{t-1})^Tx_i + (\frac{\Delta \alpha_{k,i}^t}{\lambda N})^2)$
     \STATE $\alpha_{k,i}^{t-1}=\alpha_{k,i}^{t-1} + \Delta \alpha_{k,i}^t$ 
     
 \ENDFOR
 \STATE Return $\Delta \alpha_k^t$.
 \end{algorithmic}
\end{algorithm}

\subsection{Implementation}

The exact details of the implementation are provided in the Supplementary Materials. Data are inherently non-IID in the hybrid FL case because each client stores different sections of the feature space. We emphasize that the experiments are performed in the hybrid setting where both samples and features are gathered across different clients. Homomorphic encryption is not actually performed; instead, published time benchmarks of homomorphic encryption is used to estimate the encryption time penalty which is added to the overall wall time. The regularization parameter, $\lambda$, is found by tuning via a centralized model where the value of $\lambda$ that resulted in the highest validation accuracy is employed. The resulting choices of $\lambda$ are $\lambda_{MNIST}=0.001$, $\lambda_{News20}=1\times 10^{-5}$, and $\lambda_{covtype}=5\times 10^{-5}$.

Hyperparameter tuning for federated learning is difficult because there are many competing interests such as minimizing iterations to reach a suitable solution while also minimizing the amount of computation performed on clients due to computational limits on common clients such as smartphones. Therefore, we frame this as a multiobjective optimization problem where an optimal solution must be selected from the Pareto-Optimal front. We chose to use Gray Relational Analysis to solve this \citep{GRA}. The exact metrics used are provided in the Supplementary Materials. For FedAvg, we tuned the number of local iterations of SGD performed as well as $a,b$ in the learning rate $\gamma_t=\frac{a}{b+\sqrt{t}}$. For HyFEM, we tune the aforemetnioned FedAvg hyperparameters in addition to $\mu$ which balances the two losses. For HyFDCA, we only need to tune the number of inner iterations. In each problem setting, the number of clients and fraction of available clients had different hyperparemeters tuned for that particular problem.

The plots shown use the relative loss function which is defined as $P_R = \frac{P(w^t)-P_C^*}{P_C^*}$ where $P_C^*$ is the optimal loss function value trained centrally. 
The x-axis of Figure \ref{figure3} shows relative outer iteration because each dataset requires a different number of outer iterations, but the number of iterations for each algorithm is kept consistent. Further plotting details are provided in the Supplementary Materials.

\begin{figure*}[ht!]
\begin{center}
\includegraphics[width=.88\textwidth]{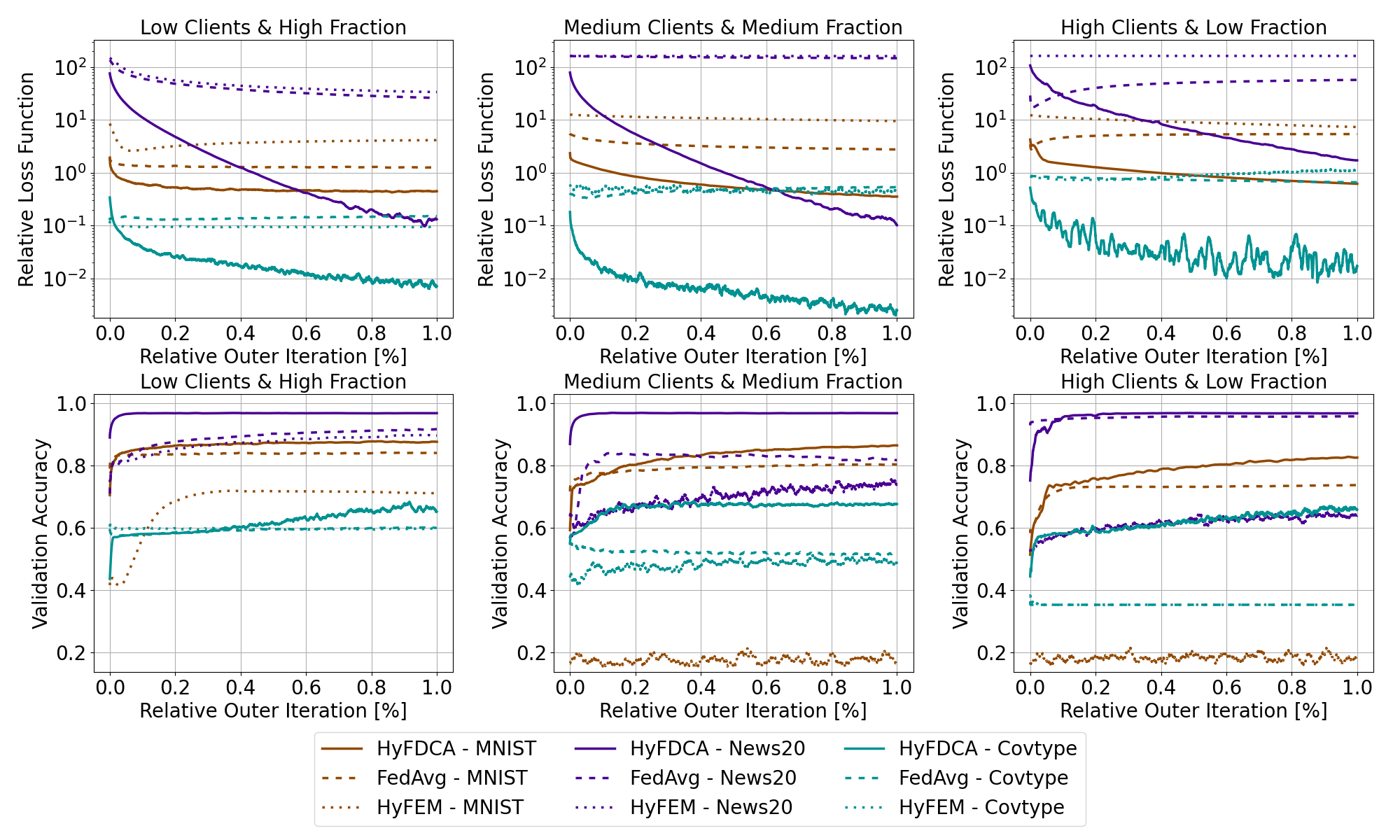}
\caption{Comparison of HyFDCA, FedAvg, \& HyFEM over constant number of outer iterations across client-fraction settings.}
\label{figure3}
\end{center}
\end{figure*}


\begin{figure*}[ht!]
\begin{center}
\includegraphics[width=.85\textwidth]{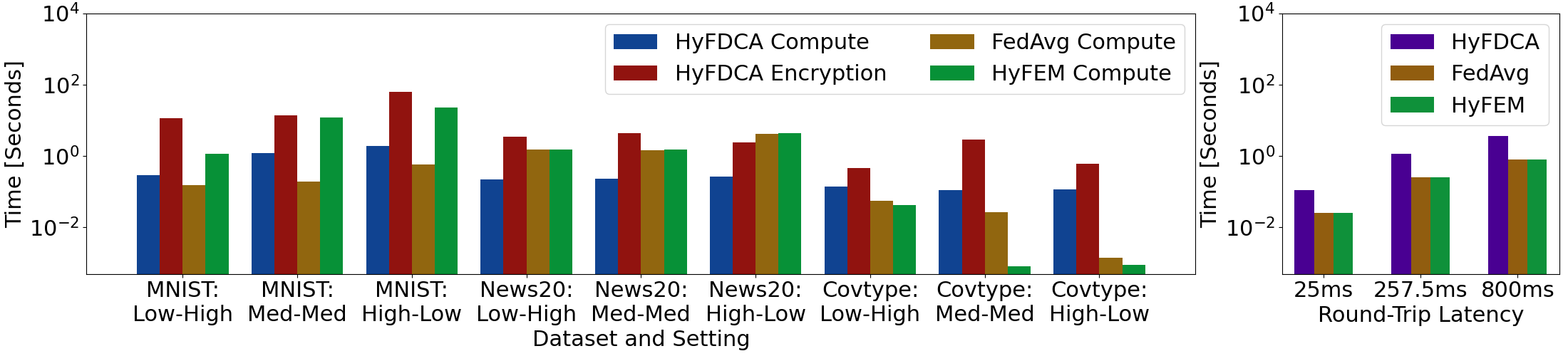}
\caption{Average time costs of components of each outer iteration for HyFDCA, FedAvg, and HyFEM.}
\label{figure4}
\end{center}
\end{figure*}

\subsection{Results}
\label{results_section}
We now discuss the results of the experiments. Due to the large number of problem settings we investigate and the various metrics, only selected plots are displayed in the main body. Complete results are included in the Supplementary Materials.

Figure \ref{figure3} compares the performance of HyFDCA with FedAvg and HyFEM over a variety of settings with respect to outer iterations. Similarly, Supplementary Materials Figure 1 compares performance with respect to time including communication latency and homomorphic encryption time. These plots correspond to varying levels of difficulty. Intuitively, a large number of clients with a low fraction of participating clients is more difficult than a small number of clients with a high fraction of participating clients. Figure \ref{figure3} and Supplementary Materials Figure 1 show that HyFDCA converges to a lower relative loss function value and a higher validation accuracy in 69 of 72 comparisons made. The poor performance of FedAvg demonstrates that algorithms designed specifically for horizontal or vertical FL cannot simply be lifted to the hybrid case. HyFEM's similarly poor performance demonstrates that its main utility is in non-convex problems with significant overlap in feature spaces across clients and where the matching of nonlinear embeddings can be utilized. Moreover, these results are indicative that, though HyFEM can converge, it may result in a substantially worse solution than if trained centrally. Finally, though HyFDCA is a significantly more complex algorithm, HyFDCA often achieves better loss and generalization in a shorter amount of both outer iterations and time even accounting for encryption and latency.


Figure \ref{figure4} shows the average time cost breakdown per outer iteration of the three algorithms. HyFDCA takes more time per outer iterations than FedAvg or HyFEM. However, the most expensive component of HyFDCA is the homomorphic encryption cost. This is expected to significantly decrease in the future as homomorphic encryption algorithms become much faster due to heavy research efforts. In addition, various methods can be employed to decrease the homomorphic encryption costs such as parallelizing the encryption/decryption of the vectors or choosing whether to encrypt the primal or the dual variables depending on the dataset.

 While HyFDCA is a more complicated algorithm involving more RTC and homomorphic encryption, the clear empirical performance gains over FedAvg and HyFEM make it superior in 69 of 72 comparisons examined. Specifically, HyFDCA converges to a lower loss value and higher validation accuracy in less overall time in 33 of 36 comparisons examined and 36 of 36 comparisons examined with respect to the number of outer iterations. Lastly, HyFDCA only requires tuning of one hyperparameter, the number of inner iterations, as opposed to FedAvg (which requires tuning three) or HyFEM (which requires tuning four). In addition to FedAvg and HyFEM being quite difficult to optimize hyperparameters in turn greatly affecting convergence, HyFDCA's single hyperparameter  allows for simpler practical implementations and hyperparameter selection methodologies.
 
 
\nocite{*}

\bibliography{aaai24}

\begin{thebibliography}{26}
\providecommand{\natexlab}[1]{#1}

\bibitem[{Adorjan(2020)}]{AWS-latency}
Adorjan, M. 2020.
\newblock {AWS} Inter-Region Latency Monitoring.
\newblock \url{https://github.com/mda590/cloudping.co.git}.

\bibitem[{Bergstra and Bengio(2012)}]{random-search}
Bergstra, J.; and Bengio, Y. 2012.
\newblock Random Search for Hyper-Parameter Optimization.
\newblock \emph{Journal of Machine Learning Research}, 13: 281–305.

\bibitem[{Chang and Lin(2011)}]{news20_covtype}
Chang, C.-C.; and Lin, C.-J. 2011.
\newblock {LIBSVM}: A Library for Support Vector Machines.
\newblock \emph{ACM Transactions on Intelligent Systems and Technology}, 2(3).

\bibitem[{Deng(2012)}]{MNIST}
Deng, L. 2012.
\newblock The mnist database of handwritten digit images for machine learning research.
\newblock \emph{IEEE Signal Processing Magazine}, 29(6): 141--142.

\bibitem[{Fan, Fang, and Friedlander(2022)}]{fedDCD}
Fan, Z.; Fang, H.; and Friedlander, M.~P. 2022.
\newblock A dual approach for federated learning.
\newblock \emph{arXiv:2201.11183}.

\bibitem[{Gentry(2009)}]{homomorphic_encryption}
Gentry, C. 2009.
\newblock Fully Homomorphic Encryption Using Ideal Lattices.
\newblock In \emph{Proceedings of the Forty-First Annual ACM Symposium on Theory of Computing}, STOC '09, 169–178.
\newblock ISBN 9781605585062.

\bibitem[{Harris et~al.(2020)}]{numpy}
Harris, C.~R.; et~al. 2020.
\newblock Array programming with NumPy.
\newblock \emph{Nature}, 585(7825): 357--362.

\bibitem[{Jaggi et~al.(2014)Jaggi, Smith, Tak\'{a}\v{c}, Terhorst, Krishnan, Hofmann, and Jordan}]{cocoa1}
Jaggi, M.; Smith, V.; Tak\'{a}\v{c}, M.; Terhorst, J.; Krishnan, S.; Hofmann, T.; and Jordan, M.~I. 2014.
\newblock Communication-Efficient Distributed Dual Coordinate Ascent.
\newblock In \emph{Proceedings of the 27th International Conference on Neural Information Processing Systems}, volume~2, 3068–3076.

\bibitem[{Li et~al.(2020{\natexlab{a}})Li, Sahu, Talwalkar, and Smith}]{fl}
Li, T.; Sahu, A.~K.; Talwalkar, A.; and Smith, V. 2020{\natexlab{a}}.
\newblock Federated Learning: Challenges, Methods, and Future Directions.
\newblock \emph{IEEE Signal Processing Magazine}, 37(3): 50--60.

\bibitem[{Li et~al.(2020{\natexlab{b}})Li, Huang, Yang, Wang, and Zhang}]{fedavg_convergence}
Li, X.; Huang, K.; Yang, W.; Wang, S.; and Zhang, Z. 2020{\natexlab{b}}.
\newblock On the Convergence of FedAvg on Non-IID Data.
\newblock In \emph{8th International Conference on Learning Representations}.

\bibitem[{Liu et~al.(2022)Liu, Zhang, Kang, Li, Chen, Hong, and Yang}]{fedbcd}
Liu, Y.; Zhang, X.; Kang, Y.; Li, L.; Chen, T.; Hong, M.; and Yang, Q. 2022.
\newblock FedBCD: A Communication-Efficient Collaborative Learning Framework for Distributed Features.
\newblock \emph{IEEE Transactions on Signal Processing}, 1--12.

\bibitem[{Lu and Ding(2020)}]{9343209}
Lu, L.; and Ding, N. 2020.
\newblock Multi-party Private Set Intersection in Vertical Federated Learning.
\newblock In \emph{2020 IEEE 19th International Conference on Trust, Security and Privacy in Computing and Communications (TrustCom)}, 707--714.

\bibitem[{Lv et~al.(2021)Lv, Cheng, Zhang, Ye, Meng, and Wang}]{lv}
Lv, B.; Cheng, P.; Zhang, C.; Ye, H.; Meng, X.; and Wang, X. 2021.
\newblock Research on Modeling of E-banking Fraud Account Identification Based on Federated Learning.
\newblock In \emph{2021 IEEE Intl Conf on Dependable, Autonomic and Secure Computing, Intl Conf on Pervasive Intelligence and Computing, Intl Conf on Cloud and Big Data Computing, Intl Conf on Cyber Science and Technology Congress (DASC/PiCom/CBDCom/CyberSciTech)}, 611--618.

\bibitem[{McMahan et~al.(2017)McMahan, Moore, Ramage, Hampson, and y~Arcas}]{fedavg}
McMahan, H.~B.; Moore, E.; Ramage, D.; Hampson, S.; and y~Arcas, B.~A. 2017.
\newblock Communication-Efficient Learning of Deep Networks from Decentralized Data.
\newblock In \emph{AISTATS}, volume~54, 1273--1282.

\bibitem[{Nathan and Klabjan(2017)}]{d3ca}
Nathan, A.; and Klabjan, D. 2017.
\newblock Optimization for Large-Scale Machine Learning with Distributed Features and Observations.
\newblock In \emph{Machine Learning and Data Mining in Pattern Recognition}, 132--146.
\newblock ISBN 978-3-319-62416-7.

\bibitem[{Overman, Blum, and Klabjan(2023)}]{overman2023primaldual}
Overman, T.; Blum, G.; and Klabjan, D. 2023.
\newblock A Primal-Dual Algorithm for Hybrid Federated Learning.
\newblock arXiv:https://arxiv.org/abs/2210.08106.

\bibitem[{Parikh and Boyd(2014)}]{block_splitting_admm}
Parikh, N.; and Boyd, S. 2014.
\newblock Block splitting for distributed optimization.
\newblock \emph{Mathematical Programming Computation}, 6(1): 77--102.

\bibitem[{Shalev-Shwartz and Zhang(2013)}]{shalev}
Shalev-Shwartz, S.; and Zhang, T. 2013.
\newblock Stochastic Dual Coordinate Ascent Methods for Regularized Loss.
\newblock \emph{Journal of Machine Learning Research}, 14(1): 567–599.

\bibitem[{Sidorov, Wei, and Ng(2022)}]{homo_enc_speed}
Sidorov, V.; Wei, E. Y.~F.; and Ng, W.~K. 2022.
\newblock Comprehensive Performance Analysis of Homomorphic Cryptosystems for Practical Data Processing.
\newblock \emph{arXiv:2202.02960}.

\bibitem[{Smith et~al.(2017)Smith, Forte, Ma, Tak\'{a}\v{c}, Jordan, and Jaggi}]{cocoa2}
Smith, V.; Forte, S.; Ma, C.; Tak\'{a}\v{c}, M.; Jordan, M.~I.; and Jaggi, M. 2017.
\newblock CoCoA: A General Framework for Communication-Efficient Distributed Optimization.
\newblock \emph{Journal of Machine Learning Research}, 18(1): 8590–8638.

\bibitem[{Telesat(2017)}]{GEO-latency}
Telesat. 2017.
\newblock Real-Time Latency: Rethinking Remote Networks.
\newblock \url{https://www.telesat.com/wp-content/uploads/2020/07/Real-Time-Latency\_HW.pdf}.

\bibitem[{T‑Mobile~USA(2022)}]{5G-latency}
T‑Mobile~USA, I. 2022.
\newblock Policies: Open Internet.
\newblock \url{https://www.t-mobile.com/responsibility/consumer-info/policies/internet-service?INTNAV=fNav:OpenInternet}.

\bibitem[{Virtanen et~al.(2020)}]{scipy}
Virtanen, P.; et~al. 2020.
\newblock SciPy 1.0: fundamental algorithms for scientific computing in {P}ython.
\newblock \emph{Nature Methods}, 17(3): 261--272.

\bibitem[{Wang and Rangaiah(2017)}]{GRA}
Wang, Z.; and Rangaiah, G.~P. 2017.
\newblock Application and Analysis of Methods for Selecting an Optimal Solution from the Pareto-Optimal Front obtained by Multiobjective Optimization.
\newblock \emph{Industrial \& Engineering Chemistry Research}, 56(2): 560--574.

\bibitem[{Yang(2013)}]{tianbao}
Yang, T. 2013.
\newblock Trading Computation for Communication: Distributed Stochastic Dual Coordinate Ascent.
\newblock In \emph{Advances in Neural Information Processing Systems}, volume~26.

\bibitem[{Zhang et~al.(2020)Zhang, Yin, Hong, and Chen}]{hyfem}
Zhang, X.; Yin, W.; Hong, M.; and Chen, T. 2020.
\newblock Hybrid Federated Learning: Algorithms and Implementation.
\newblock In \emph{NeurIPS-SpicyFL 2020}.

\end{thebibliography}

\onecolumn
\section{Convergence Proofs}
\begin{lemma}
\label{1}
Let $l_i:\mathbb{R} \rightarrow \mathbb{R}$ be an $L$-Lipschitz continuous function. Then for any real value $\alpha$ with $|\alpha|>L$ we have that $l_i^*(\alpha)=\infty$
\end{lemma}
\begin{proof}
Proof provided in Lemma 21 in \citep{shalev}.
\end{proof}

\subsection{Proof of Theorem 1}
\begin{proof}
For simplicity, the global dual and primal variables are denoted as: $w=w_0$ and $\alpha_i=\alpha_{0,i}$. We also frequently use the following relationships.
\begin{enumerate}
    \item The map from dual to primal is $w = \frac{1}{\lambda N}\sum_{i=1}^N\alpha_i x_i = \textbf{A}\alpha$.
    \item $\alpha_i^t= \alpha_i^{t-1}+\gamma_t\Delta\alpha_i^t$
    \item Equalities $\alpha_i^{t-1}=\alpha_{b,i}^{t-1}=\frac{1}{|\mathcal{B}_i|}\sum_{b\in\mathcal{B}_i}\alpha_{b,i}^{t-1}$ hold because the aggregated dual variables on the server are sent back to the clients every iteration.
    \item Following from Cauchy-Schwarz, we have $||\sum_{i=1}^N z_i||^2 \leq N\sum_{i=1}^N ||z_i||^2$.
\end{enumerate}

 Starting with the difference in dual objective after one outer iteration,  we have
\begin{align*}
    N[D(\alpha^t) - D(\alpha^{t-1})] = \underbrace{[-\sum_{i=1}^Nl_i^*(-\alpha_i^t) - \frac{\lambda N}{2}||\textbf{A}\alpha^t||^2]}_{A} - \underbrace{[-\sum_{i=1}^Nl_i^*(-\alpha_i^{t-1}) - \frac{\lambda N}{2}||\textbf{A}\alpha^{t-1}||^2]}_{B}.
\end{align*}
We rewrite $A$ as
\begin{align*}
    A = \underbrace{-\sum_{i=1}^N l_i^*(-\alpha_i^t)}_{A_1} \underbrace{- \frac{\lambda N}{2}||\textbf{A}\alpha^t||^2}_{A_2}
\end{align*}

and then bound $A_1$ as
\begin{align*}
    A_1 &=-\sum_{i=1}^N l_i^*(-\alpha_i^{t-1}-\gamma_t \Delta \alpha_i^t)\\
    &= -\sum_{i=1}^N l_i^*(-(1-\gamma_t)\alpha_i^{t-1} - \gamma_t(\alpha_i^{t-1}+\Delta \alpha_i^{t}))\\
    &\geq -\sum_{i=1}^N[\gamma_t l_i^*(-\alpha_i^{t-1}-\Delta \alpha_i^t)+(1-\gamma_t)l_i^*(-\alpha_i^{t-1})]
\end{align*}
\clearpage

and rewrite $A_2$ as
\begin{align*}
    A_2 &= -\frac{\lambda N}{2}||\textbf{A}\alpha^t||^2\\
    &= -\frac{\lambda N}{2}||\textbf{A}[\alpha^{t-1}+\gamma_t \Delta \alpha^t] ||^2\\
    &=-\frac{\lambda N}{2}||w^{t-1} + \gamma_t \textbf{A}\Delta \alpha^t||^2\\
    &=-\frac{\lambda N}{2}[||w^{t-1}||^2 + 2\gamma_t(w^{t-1})^T\textbf{A}\Delta \alpha^t + \gamma_t^2||\textbf{A}\Delta \alpha^t||^2].
\end{align*}

Our bound on $A$ is
\begin{align*}
    A \geq -\sum_{i=1}^N [\gamma_t l_i^*(-\alpha_i^{t-1}-\Delta \alpha_i^{t}) + (1-\gamma_t)l_i^*(-\alpha_i^{t-1})] - \frac{\lambda N}{2}[||w^{t-1}||^2+2\gamma_t(w^{t-1})^T\textbf{A}\Delta \alpha^t + \gamma_t^2||\textbf{A}\Delta \alpha^t||^2].
\end{align*}

Expression $B$ appears in the RHS, so we can simplify
\begin{align*}
    A-B \geq -\sum_{i=1}^N [\gamma_t l_i^*(-\alpha_i^{t-1}-\Delta \alpha_i^{t})  -\gamma_tl_i^*(-\alpha_i^{t-1})] - \frac{\lambda N}{2}[2\gamma_t(w^{t-1})^T\textbf{A}\Delta \alpha^t + \gamma_t^2||\textbf{A}\Delta \alpha^t||^2].
\end{align*}

Next, we re-write the linear operator $\textbf{A}$ as a summation over samples
\begin{align*}
    A-B &\geq -\sum_{i=1}^N [\gamma_t l_i^*(-\alpha_i^{t-1}-\Delta \alpha_i^{t})  -\gamma_tl_i^*(-\alpha_i^{t-1})] - \frac{\lambda N}{2}[2\gamma_t(w^{t-1})^T\sum_{i=1}^N\frac{1}{\lambda N}x_i\Delta \alpha_i^t + \gamma_t^2||\sum_{i=1}^N\frac{1}{\lambda N}x_i\Delta \alpha_i^t||^2]\\
    &= -\sum_{i=1}^N [\gamma_t l_i^*(-\alpha_i^{t-1}-\Delta \alpha_i^{t})  -\gamma_tl_i^*(-\alpha_i^{t-1}) +\gamma_t(w^{t-1})^Tx_i\Delta \alpha_i^t] - \frac{\lambda N}{2}[\gamma_t^2||\sum_{i=1}^N\frac{1}{\lambda N}x_i\Delta \alpha_i^t||^2]\\
    &\geq -\sum_{i=1}^N [\gamma_t l_i^*(-\alpha_i^{t-1}-\Delta \alpha_i^{t})  -\gamma_tl_i^*(-\alpha_i^{t-1}) +\gamma_t(w^{t-1})^Tx_i\Delta \alpha_i^t] - \frac{\lambda N}{2}[\frac{\gamma_t^2}{\lambda^2 N}\sum_{i=1}^N||x_i\Delta \alpha_i^t||^2].\\
\end{align*}

Without loss of generality, we assume our data is normalized such that $||x_i||\leq1$, which yields
\begin{align*}
    A-B&\geq -\sum_{i=1}^N [\gamma_t l_i^*(-\alpha_i^{t-1}-\Delta \alpha_i^{t})  -\gamma_tl_i^*(-\alpha_i^{t-1}) +\gamma_t(w^{t-1})^Tx_i\Delta \alpha_i^t + \frac{\gamma_t^2}{2\lambda}(\Delta \alpha_i^t)^2 ].\\
\end{align*}

Our dual updates are defined as $\Delta \alpha_i^t=\frac{1}{|\mathcal{B}_i|}\sum_{b \in \mathcal{B}_i}\Delta \alpha_{b,i}^t$, which is just the mean of the dual variable updates on each client that contains sample $i$. In turn we have

\begin{align*}
    A-B&\geq -\sum_{i=1}^N [\gamma_t l_i^*(-\alpha_i^{t-1}-\frac{1}{|\mathcal{B}_i|}\sum_{b \in \mathcal{B}_i}\Delta \alpha_{b,i}^t)  -\gamma_tl_i^*(-\alpha_i^{t-1}) +\gamma_t(w^{t-1})^Tx_i(\frac{1}{|\mathcal{B}_i|}\sum_{b \in \mathcal{B}_i}\Delta \alpha_{b,i}^t) + \frac{\gamma_t^2}{2\lambda}(\frac{1}{|\mathcal{B}_i|}\sum_{b \in \mathcal{B}_i}\Delta \alpha_{b,i}^t)^2 ]\\
    &= -\sum_{i=1}^N [\gamma_t l_i^*(-\frac{1}{|\mathcal{B}_i|}\sum_{b \in \mathcal{B}_i}(\alpha_i^{t-1} + \Delta \alpha_{b,i}^t))  -\gamma_tl_i^*(-\alpha_i^{t-1}) +\gamma_t(w^{t-1})^Tx_i(\frac{1}{|\mathcal{B}_i|}\sum_{b \in \mathcal{B}_i}\Delta \alpha_{b,i}^t) + \frac{\gamma_t^2}{2\lambda}(\frac{1}{|\mathcal{B}_i|}\sum_{b \in \mathcal{B}_i}\Delta \alpha_{b,i}^t)^2 ].
\end{align*}
By convexity of $l_i^*$ we get
\begin{align*}
    A-B&\geq -\sum_{i=1}^N [\frac{\gamma_t}{|\mathcal{B}_i|}\sum_{b \in \mathcal{B}_i} l_i^*(-(\alpha_i^{t-1} + \Delta \alpha_{b,i}^t))  -\gamma_tl_i^*(-\alpha_i^{t-1}) +\frac{\gamma_t}{|\mathcal{B}_i|}\sum_{b \in \mathcal{B}_i}(w^{t-1})^Tx_i\Delta \alpha_{b,i}^t + \frac{\gamma_t^2}{2\lambda |\mathcal{B}_i|}\sum_{b \in \mathcal{B}_i}(\Delta \alpha_{b,i}^t)^2 ]\\
    &=-\gamma_t\sum_{i=1}^N \frac{1}{|\mathcal{B}_i|}  \sum_{b \in \mathcal{B}_i} [l_i^*(-(\alpha_i^{t-1} + \Delta \alpha_{b,i}^t))  - l_i^*(-\alpha_i^{t-1}) +(w^{t-1})^Tx_i\Delta \alpha_{b,i}^t + \frac{\gamma_t}{2\lambda}(\Delta \alpha_{b,i}^t)^2 ].
\end{align*}

The stochastic element of this algorithm is in randomly determining which $\Delta \alpha_{b,i}$ are nonzero. If a particular $\Delta \alpha_{b,i}$ is zero, then that term in the double series in the right hand side of the inequality is zero. If that $\Delta \alpha_{b,i}$ is nonzero, then the local dual method chooses the $\Delta \alpha_{b,i}$ as follows. We choose our local dual method such that $\Delta \alpha_{b,i}^t=s_{b,i} c_b (u_{i}^{t-1}-\alpha_{b,i}^{t-1})$ where $s_{b,i} \in [0,1]$ and $u_{i}^{t-1}\in\partial l_i(x_{i}^Tw_0^{t-1})$. 
We define our local dual method to find $s_{b,i}$ as follows
\begin{align*}
    s_{b,i}=\argmax_{s\in [0,1]} \{ -l_i^*(-(\alpha_{b,i}^{t-1} + s c_b(u_{i}^{t-1}-\alpha_{b,i}^{t-1}))) -s c_b(w^{t-1})^Tx_i(u_{i}^{t-1}-\alpha_{b,i}^{t-1}) - \frac{\gamma_t}{2\lambda}(s c_b(u_{i}^{t-1}-\alpha_{b,i}^{t-1}))^2 \}.
\end{align*}

Now we must find the probability that a particular $\alpha_{b,i}$ is updated. This is $H/N_b$ where $N_b$ is the number of samples on client $b$. So we take the expectation over these client/sample pairs while conditioning on the previous state, $\alpha_i^{t-1}$, to obtain

\begin{align*}
    \mathbb{E}[A-B|\alpha^{t-1}] &\geq -\gamma_t H\sum_{i=1}^N \frac{1}{|\mathcal{B}_i|} \sum_{b \in \mathcal{B}_i} \frac{1}{N_b}[l_i^*(-(\alpha_i^{t-1} + s_{b,i}c_b (u_{i}^{t-1}-\alpha_{b,i}^{t-1})))  - l_i^*(-\alpha_i^{t-1})\\ &+(w^{t-1})^Tx_is_{b,i} c_b(u_{i}^{t-1}-\alpha_{b,i}^{t-1}) + \frac{\gamma_t}{2\lambda}(s_{b,i}c_b(u_{i}^{t-1}-\alpha_{b,i}^{t-1}))^2 ].
\end{align*}

Due to the definition of the LocalDualMethod, the inequality holds for any choice of $s_t\in[0,1]$, and thus
\begin{align*}
    \mathbb{E}[A-B|\alpha^{t-1}] &\geq -\gamma_t H\sum_{i=1}^N \frac{1}{|\mathcal{B}_i|}  \sum_{b \in \mathcal{B}_i} \frac{1}{N_b} [l_i^*(-(\alpha_i^{t-1} + s_{t}c_b(u_{i}^{t-1}-\alpha_{b,i}^{t-1})))  - l_i^*(-\alpha_i^{t-1})\\ &+(w^{t-1})^Tx_is_{t}c_b(u_{i}^{t-1}-\alpha_{b,i}^{t-1}) + \frac{\gamma_t}{2\lambda}(s_{t}c_b(u_{i}^{t-1}-\alpha_{b,i}^{t-1}))^2 ].
\end{align*}

By convexity of $l_i^*$ and because $c_b=\frac{N_b}{N}\leq 1$, we conclude
\begin{align*}
    \mathbb{E}[A-B|\alpha^{t-1}] &\geq -\gamma_t H\sum_{i=1}^N \frac{1}{|\mathcal{B}_i|}  \sum_{b \in \mathcal{B}_i} \frac{1}{N_b} [(1-s_{t}c_b)l_i^*(-\alpha_{b,i}^{t-1}) + s_{t}c_b l_i^*(-u_{i}^{t-1})  - l_i^*(-\alpha_i^{t-1})\\ &+(w^{t-1})^Tx_is_{t}c_b(u_{i}^{t-1}-\alpha_{b,i}^{t-1}) + \frac{\gamma_t}{2\lambda}(s_{t}c_b(u_{i}^{t-1}-\alpha_{b,i}^{t-1}))^2 ]
\end{align*}

and thus from $\alpha_{b,i}^{t-1} = \alpha_i^{t-1}$ and inputting $c_b=\frac{N_b}{N}$, we have that
\begin{align*}
    \mathbb{E}[A-B|\alpha^{t-1}] &\geq -\gamma_t H\sum_{i=1}^N \frac{1}{|\mathcal{B}_i|}  \sum_{b \in \mathcal{B}_i} \frac{1}{N} [-s_tl_i^*(-\alpha_{i}^{t-1}) + s_t l_i^*(-u_{i}^{t-1}) +s_t(w^{t-1})^Tx_i(u_{i}^{t-1}-\alpha_{i}^{t-1}) + \frac{\gamma_t s_t^{2} c_b}{2\lambda}(u_{i}^{t-1}-\alpha_{i}^{t-1})^2 ].
\end{align*}

We can safely remove $c_b\leq 1$ from the final term without changing the bound to get
\begin{align*}
    \mathbb{E}[A-B|\alpha^{t-1}] &\geq -\gamma_t H\sum_{i=1}^N \frac{1}{|\mathcal{B}_i|}  \sum_{b \in \mathcal{B}_i} \frac{1}{N} [-s_tl_i^*(-\alpha_{i}^{t-1}) + s_t l_i^*(-u_{i}^{t-1}) +s_t(w^{t-1})^Tx_i(u_{i}^{t-1}-\alpha_{i}^{t-1}) + \frac{\gamma_t s_t^{2}}{2\lambda}(u_{i}^{t-1}-\alpha_{i}^{t-1})^2 ].
\end{align*}

Now from convex conjugates we know that $l_i(x_i^Tw^{t-1})=-l_i^*(-u_i^{t-1})-u_i^{t-1}x_i^Tw^{t-1}$, and thus
\begin{align*}
    \mathbb{E}[A-B|\alpha^{t-1}] &\geq -\gamma_t H\sum_{i=1}^N \frac{1}{|\mathcal{B}_i|}  \sum_{b \in \mathcal{B}_i} \frac{1}{N} [-s_tl_i^*(-\alpha_{i}^{t-1}) -s_t l_i(x_i^Tw^{t-1}) -s_t(w^{t-1})^Tx_i\alpha_{i}^{t-1} + \frac{\gamma_t s_t^{2}}{2\lambda}(u_{i}^{t-1}-\alpha_{i}^{t-1})^2 ].
\end{align*}

We simplify the expression to obtain 
\begin{align}
    \label{canceled}
    \mathbb{E}[A-B|\alpha^{t-1}] &\geq - \frac{\gamma_t H}{N}\sum_{i=1}^N [-s_tl_i^*(-\alpha_{i}^{t-1}) -s_t l_i(x_i^Tw^{t-1}) -s_t(w^{t-1})^Tx_i\alpha_{i}^{t-1} + \frac{\gamma_t s_t^{2}}{2\lambda}(u_{i}^{t-1}-\alpha_{i}^{t-1})^2 ].
\end{align}

The primal-dual gap at iteration $t-1$ is defined as 
\begin{align*}
    P(w^{t-1})-D(\alpha^{t-1}) &= (\frac{\lambda}{2}||w^{t-1}||^2 + \frac{1}{N}\sum_{i=1}^Nl_i((w^{t-1})^Tx_i)) - (-\frac{\lambda}{2}||\frac{1}{\lambda N}\sum_{i=1}^N \alpha_i^{t-1}x_i||^2-\frac{1}{N}\sum_{i=1}^{N}l_i^*(-\alpha_i^{t-1}))\\
    &= \frac{1}{N}\sum_{i=1}^{N}l_i^*(-\alpha_i^{t-1}) + \frac{1}{N}\sum_{i=1}^Nl_i((w^{t-1})^Tx_i) + \lambda||w^{t-1}||^2\\
    &= \frac{1}{N}\sum_{i=1}^{N}l_i^*(-\alpha_i^{t-1}) + \frac{1}{N}\sum_{i=1}^Nl_i((w^{t-1})^Tx_i) + \lambda (w^{t-1})^Tw^{t-1}\\
    &=\frac{1}{N}\sum_{i=1}^{N}l_i^*(-\alpha_i^{t-1}) + \frac{1}{N}\sum_{i=1}^Nl_i((w^{t-1})^Tx_i) + \lambda (w^{t-1})^T(\frac{1}{\lambda N}\sum_{i=1}^N\alpha_i^{t-1} x_i)\\
    &=\frac{1}{N}\sum_{i=1}^N[l_i((w^{t-1})^Tx_i) + l_i^*(-\alpha_i^{t-1})+\alpha^{t-1}_ix_i^Tw^{t-1}].
\end{align*}

We plug this primal-dual gap into (\ref{canceled}) to obtain
\begin{align*}
    \mathbb{E}[A-B|\alpha^{t-1}] &\geq s_t \gamma_t H [P(w^{t-1})-D(\alpha^{t-1})] -\frac{\gamma_t^2s_t^2 H}{2N\lambda}\sum_{i=1}^N (u_{i}^{t-1}-\alpha_{i}^{t-1})^2
\end{align*}

and after dividing by $N$
\begin{align*}
    \mathbb{E}[D(\alpha^t)-D(\alpha^{t-1})|\alpha^{t-1}] &\geq \frac{s_t \gamma_t H}{N} [P(w^{t-1})-D(\alpha^{t-1})] -\frac{\gamma_t^2s_t^2 H}{2N^2\lambda}\sum_{i=1}^N (u_{i}^{t-1}-\alpha_{i}^{t-1})^2.
\end{align*}

By assumption, $\gamma_t=1$ and setting $G^{t-1}=\frac{1}{2N\lambda}\sum_{i=1}^N(u_i^{t-1}-\alpha_i^{t-1})^2$, we get
\begin{align}
\label{expec}
    \mathbb{E}[D(\alpha^t)-D(\alpha^{t-1})|\alpha^{t-1}] &\geq \frac{s_t H}{N} [P(w^{t-1})-D(\alpha^{t-1})] -\frac{s_t^2 H}{N}G^{t-1}.
\end{align}

We know that $\varepsilon_D^{t-1}=D(\alpha^*)-D(\alpha^{t-1})\leq P(w^{t-1})-D(\alpha^{t-1})$ and $D(\alpha^t)-D(\alpha^{t-1})=\varepsilon_D^{t-1}-\varepsilon_D^t$. It must also hold that $\mathbb{E}[D(\alpha^t)-D(\alpha^{t-1})|\alpha^{t-1}]=\mathbb{E}[\varepsilon_D^{t-1}-\varepsilon_D^t | \alpha^{t-1}] = \varepsilon_D^{t-1} - \mathbb{E}[\varepsilon_D^t | \alpha^{t-1}]$. Together with (\ref{expec}) this yields

\begin{align*}
    \varepsilon_D^{t-1} - \mathbb{E}[\varepsilon_D^t | \alpha^{t-1}] &\geq \frac{s_t H}{N} \varepsilon_D^{t-1} -\frac{s_t^2 H}{N}G^{t-1}.
\end{align*}

Let $G^{t-1}\leq G$. Then,

\begin{align*}
    \mathbb{E}[\varepsilon_D^t | \alpha^{t-1}] &\leq (1-\frac{s_t H}{N}) \varepsilon_D^{t-1} +\frac{s_t^2 H}{N}G.
\end{align*}

Taking the expectation of both sides and using the law of iterated expectation we obtain
\begin{align*}
    \mathbb{E}[\varepsilon_D^t] &\leq (1-\frac{s_t H}{N}) \mathbb{E}[\varepsilon_D^{t-1}] +\frac{s_t^2 H}{N}G.
\end{align*}

Finally we need to bound G. From Lemma \ref{1} we know that each $|\alpha_{b,i}|<L$ due to the choice of LocalDualMethod. Furthermore, $|u_i|<L$ because $l_i$ are $L$-Lipschitz. We conclude that
\begin{equation*}
    G = \frac{2L^2}{\lambda}.
\end{equation*}

This concludes the proof.

\end{proof}

\subsection{Proof of Theorem 2}
\begin{proof}
This is proved by using induction on $t$. It also uses the fact that $1+x \leq e^x$. Let $s_t=1$ for $t\leq t_0$ and $s_t=\frac{1}{1+\frac{H}{2N}(t-t_0)}$ for $t>t_0$.

The base case of $t=t_0$ follows as 

\begin{align*}
    \mathbb{E}[\varepsilon_D^t] &\leq (1-\frac{H}{N})^t \mathbb{E}[\varepsilon_D^0] + \frac{H}{N} G\sum_{\tau=1}^{t-1}(1-\frac{H}{N})^\tau\\
    &\leq (\exp(\frac{-H}{N}))^t\mathbb{E}[\varepsilon_D^0] + \frac{H}{N}G(\frac{1-(1-\frac{H}{N})^{t-1}}{1-(1-\frac{H}{N})})\\
    &\leq (\exp(\frac{-H}{N}\log(\mathbb{E}[\varepsilon_D^0]/G)))\mathbb{E}[\varepsilon_D^0] + G\underbrace{(1-(1-\frac{H}{N})^{t-1})}_{\leq 1}\\
    &\leq (\frac{G}{\varepsilon_D^0})^\frac{H}{N}\mathbb{E}[\varepsilon_D^0] + G\\
    &\leq G+G\\
    &=2G.
\end{align*}
We now assume eq. 4 (from main paper) holds for some $t$, and we need to prove that it also holds for $t+1$. We start with
\begin{align*}
    \mathbb{E}[\varepsilon_D^{t+1}] &\leq (1-\frac{sH}{N})\mathbb{E}[\varepsilon_D^t]+\frac{s^2H}{N}G\\
    &\leq (1-\frac{sH}{N})(\frac{2G}{1+\frac{H}{2N}(t-t_0)})+\frac{s^2H}{N}G\\
    &=(1-\frac{H/N}{1+\frac{H}{2N}(t-t_0)})(\frac{2G}{1+\frac{H}{2N}(t-t_0)})+\frac{HG/N}{(1+\frac{H}{2N}(t-t_0))^2}\\
    &=2G\underbrace{\frac{1+\frac{H}{2N}(t-t_0)-\frac{H}{2N}}{(1+\frac{H}{2N}(t-t_0))^2}}_{U}.
\end{align*}

We bound $U$ by
\begin{align*}
    U&=\frac{1+\frac{H}{2N}(t-t_0)-\frac{H}{2N}}{(1+\frac{H}{2N}(t-t_0))^2}\\
    &=\frac{1}{1+\frac{H}{2N}(t+1-t_0)}\frac{(1+\frac{H}{2N}(t+1-t_0))(1+\frac{H}{2N}(t-1-t_0))}{(1+\frac{H}{2N}(t-t_0))^2}\\
    &=\frac{1}{1+\frac{H}{2N}(t+1-t_0)}\frac{1+\frac{H}{N}(t-t_0) + \frac{H^2}{4N^2}((t-t_0)^2-1)}{1+\frac{H}{N}(t-t_0)+\frac{H^2}{4N^2}(t-t_0)^2}\\
    &=\frac{1}{1+\frac{H}{2N}(t+1-t_0)}(1-\frac{H^2}{4N^2(1+\frac{H}{N}(t-t_0)+\frac{H^2}{4N^2}(t-t_0)^2)})\\
    &=\frac{1}{1+\frac{H}{2N}(t+1-t_0)}\underbrace{(1-\frac{H^2}{4N^2(1+\frac{H}{2N}(t-t_0))^2})}_{\leq 1}\\
    &\leq \frac{1}{1+\frac{H}{2N}(t+1-t_0)}.
\end{align*}

Therefore,
\begin{equation}
    \mathbb{E}[\varepsilon_D^{t+1}]\leq \frac{2G}{1+\frac{H}{2N}(t+1-t_0)}
\end{equation}

completing the proof.
\end{proof}

\subsection{Proof of Theorem 3}
\begin{proof}
This proof follows the proof of Theorem 1 closely. We first note some important facts that are key to the proof for the horizontal case.
\begin{enumerate}
    \item We have $Q=|\mathcal{B}_i|=1$, therefore each dual variable $\alpha_i$ only belongs to a single client. We simplify the notation by omitting the dual variable subscript pertaining to the client, e.g. we write $\alpha_i^t$ instead of $\alpha_{b,i}^t$.
    \item The inner product for sample $i$, $x_i^Tw^t$ can be found completely on the single client and no aggregation by the server is needed.
    \item The relationship $w^{t-1}=\frac{1}{\lambda N}\sum_{i=1}^N\alpha_i^{t-1} x_i$ still holds. Because some of the $\alpha_i$ have not been updated in iteration $t-1$, the algorithm has stored the contributions of each component in $\hat{w_{0,k,m}}$ of the sum so that $w^{t-1}$ can be exactly found.

\end{enumerate}
From the proof of Theorem 1, we have
\begin{align*}
    A-B&\geq -\sum_{i=1}^N [\gamma_t l_i^*(-\alpha_i^{t-1}-\Delta \alpha_i^{t})  -\gamma_tl_i^*(-\alpha_i^{t-1}) +\gamma_t(w^{t-1})^Tx_i\Delta \alpha_i^t + \frac{\gamma_t^2}{2\lambda}(\Delta \alpha_i^t)^2 ].\\
\end{align*}
In the hybrid case, we compute $\Delta \alpha_i^t = \frac{1}{Q}\sum_{b\in\mathcal{B}_i} \Delta \alpha_{b,i}^t$. However, in this case, this can be simplified due to $Q=|\mathcal{B}_i|=1$.

Now we must find the probability that a particular $\alpha_{i}$ is updated. This is probability $PH/N$, because there are a total of $N$ samples and there are $PH$ total expected updates per outer iteration. So taking the expectation over these available clients while conditioning on the previous state, $\alpha^{t-1}$, yields
\begin{align*}
    \mathbb{E}[A-B|\alpha^{t-1}]&\geq -\frac{\gamma_t PH}{N}\sum_{i=1}^N [ l_i^*(-\alpha_i^{t-1}-\Delta \alpha_i^{t})  -l_i^*(-\alpha_i^{t-1}) +(w^{t-1})^Tx_i\Delta \alpha_i^t + \frac{\gamma_t}{2\lambda}(\Delta \alpha_i^t)^2 ].\\
\end{align*}

Due to the choice of LocalDualMethod, for any $s_t \in [0,1]$, we obtain
\begin{align*}
    \mathbb{E}[A-B|\alpha^{t-1}]&\geq  -\frac{PH\gamma_t}{N}\sum_{i=1}^N [- s_tl_i^*(-\alpha_i^{t-1})+ s_tl_i^*(-u_i^{t-1})  +(w^{t-1})^Tx_is_t(u_i^{t-1}-\alpha_i^{t-1}) + \frac{\gamma_t}{2\lambda}s_t^2(u_i^{t-1}-\alpha_i^{t-1})^2 ].
\end{align*}

In the same vein as the in the proof of Theorem 1, we derive
\begin{align*}
    \mathbb{E}[A-B|\alpha^{t-1}] &\geq -\frac{PH\gamma_t}{N}\sum_{i=1}^N [- s_tl_i^*(-\alpha_i^{t-1}) -s_tl_i(x_i^Tw^{t-1})  -s_t(w^{t-1})^Tx_i\alpha_i^{t-1} + \frac{\gamma_t}{2\lambda}s_t^2(u_i^{t-1}-\alpha_i^{t-1})^2 ]
\end{align*}

and
\begin{align*}
    \mathbb{E}[A-B|\alpha^{t-1}] &\geq PH\gamma_ts_t[P(w^{t-1})-D(\alpha^{t-1})] -\frac{PH\gamma_t^2}{2\lambda N}s_t^2(u_i^{t-1}-\alpha_i^{t-1})^2 ]
\end{align*}

which in turn yields

\begin{align*}
    \mathbb{E}[D(\alpha^t)-D(\alpha^{t-1})|\alpha^{t-1}] &\geq \frac{PHs_t}{N}[P(w^{t-1})-D(\alpha^{t-1})] -\frac{PHs_t^2}{N}G^{t-1}.
\end{align*}

The rest of the proof is the same as in Theorem 1 and 2.
\end{proof}

\subsection{Proof of Theorem 4}
\begin{proof}
We first note some important points.
\begin{enumerate}
    \item If set $c$ is selected for updates, then for all $b\in \mathcal{B}_c$, we have $\alpha_{b,i}^{t-1} = \alpha_i^{t-1}$ because all clients in that set are sent the global dual variables from the server before performing local updates. 
    \item The dot product scalar, $x_i^Tw$, that is calculated and passed from the server at the start of iteration $t$ is representative of the true $x_i^Tw^{t-1}$. This is because if the feature indices are updated on a particular client, then that component of the dot product is computed and sent to the server and the server stores the components of the dot product from feature indices that were not updated.
    \item We frequently use Lemma \ref{1} and the choice of the LocalDualMethod to bound $|\alpha_i|<L$. Furthermore, $|u_i|<L$ because $l_i$ are $L$-Lipschitz.
\end{enumerate}

We examine the change in the dual objective after one outer iteration. We assume that $t\geq C$. We start with
\begin{align*}
   N[D(\alpha^{t}) - D(\alpha^{t-1})]  \geq \underbrace{[-\sum_{i=1}^N l_i^*(-\alpha_i^t)-\frac{\lambda N}{2}||\textbf{A}\alpha^t||^2]}_{A} - \underbrace{[-\sum_{i=1}^N l_i^*(-\alpha_i^{t-1})-\frac{\lambda N}{2}||\textbf{A}\alpha^{t-1}||^2]}_{B}.
\end{align*}

We first examine A by
\begin{align*}
    A &= -\sum_{i=1}^N l_i^*(-\alpha_i^t)-\frac{\lambda N}{2}||\textbf{A}\alpha^t||^2\\
    &= -\sum_{i=1}^Nl_i^*(-\alpha_i^{t-1}-\Delta \alpha_i^t)-\frac{\lambda N}{2}||\textbf{A}\alpha^{t-1} + \textbf{A}\Delta \alpha^t||^2\\
    &=-\sum_{i=1}^Nl_i^*(-\alpha_i^{t-1}-\Delta \alpha_i^t) - \frac{\lambda N}{2}[||\textbf{A}\alpha^{t-1}||^2 + 2(\textbf{A}\alpha^{t-1})^T(\textbf{A}\Delta \alpha^t)+ ||\textbf{A}\Delta \alpha^t||^2].
\end{align*}

Then we have
\begin{align*}
    A-B &\geq -\sum_{i=1}^N [l_i^*(-\alpha_i^{t-1}-\Delta \alpha_i^t) - l_i^*(-\alpha_i^{t-1})] -\lambda N(\textbf{A}\alpha^{t-1})^T(\textbf{A}\Delta \alpha^t) - \frac{\lambda N}{2}||\textbf{A}\Delta \alpha^t||^2\\
    &\geq -\sum_{i=1}^N [l_i^*(-\alpha_i^{t-1}-\Delta \alpha_i^t) - l_i^*(-\alpha_i^{t-1})] -\sum_{i=1}^N (\textbf{A}\alpha^{t-1})^T\Delta \alpha_i^t x_i - \frac{1}{2\lambda}\sum_{i=1}^N (\Delta \alpha_i^t)^2\\
    &=-\sum_{i=1}^N [l_i^*(-\alpha_i^{t-1}-\Delta \alpha_i^t)-l_i^*(-\alpha_i^{t-1})+(\textbf{A}\alpha^{t-1})^T\Delta \alpha_i^t x_i +\frac{1}{2\lambda} (\Delta \alpha_i^t)^2].
\end{align*}

We assume that at iteration $t$, a set $c$ is selected and the clients in that set are updated. We have that $\Delta \alpha_i^t = \frac{\gamma_t}{Q/C}\sum_{b \in \mathcal{B}_c}\Delta \alpha_{b,i}^{t}$ and
\begin{align*}
    A-B &\geq -\sum_{i=1}^N [l_i^*(-\alpha_i^{t-1}-\frac{\gamma_t}{Q/C}\sum_{b \in \mathcal{B}_c}\Delta \alpha_{b,i}^{t})-l_i^*(-\alpha_i^{t-1})+(\textbf{A}\alpha^{t-1})^T(\frac{\gamma_t}{Q/C}\sum_{b \in \mathcal{B}_c}\Delta \alpha_{b,i}^{t}) x_i +\frac{1}{2\lambda} (\frac{\gamma_t}{Q/C}\sum_{b \in \mathcal{B}_c}\Delta \alpha_{b,i}^{t})^2].
\end{align*}
Using convexity we have that
\begin{align*}
    l_i^*(-\alpha_i^{t-1}-\frac{\gamma_t}{Q/C}\sum_{b \in \mathcal{B}_c}\Delta \alpha_{b,i}^{t}) &= l_i^*(-\frac{1}{Q/C}\sum_{b \in \mathcal{B}_c}\alpha_i^{t-1}-\frac{\gamma_t}{Q/C}\sum_{b \in \mathcal{B}_c}\Delta \alpha_{b,i}^{t})\\
    &\leq \frac{1}{Q/C} \sum_{b\in \mathcal{B}_c} l_i^*(-\alpha_i^{t-1}-\gamma_t\Delta \alpha_{b,i}^{t}).\\
\end{align*}

Returning to $A-B$ we derive
\begin{align*}
    A-B &\geq -\sum_{i=1}^N [\frac{1}{Q/C} \sum_{b\in \mathcal{B}_c} l_i^*(-\alpha_i^{t-1}-\gamma_t\Delta \alpha_{b,i}^{t}) - l_i^*(-\alpha_i^{t-1})+\frac{\gamma_t}{Q/C}\sum_{b \in \mathcal{B}_c}(\textbf{A}\alpha^{t-1})^T(\Delta \alpha_{b,i}^{t}) x_i +\frac{1}{2\lambda} \frac{1}{Q/C}\sum_{b \in \mathcal{B}_c}(\gamma_t\Delta \alpha_{b,i}^{t})^2]\\
    &=-\frac{1}{Q/C} \sum_{i=1}^N \sum_{b\in\mathcal{B}_c} [ l_i^*(-\alpha_i^{t-1}-\gamma_t\Delta \alpha_{b,i}^{t}) - l_i^*(-\alpha_i^{t-1})+\gamma_t(\textbf{A}\alpha^{t-1})^T(\Delta \alpha_{b,i}^{t}) x_i +\frac{1}{2\lambda} (\gamma_t\Delta \alpha_{b,i}^{t})^2].
\end{align*}

The stochastic component of this algorithm is deciding if a particular sample $i$ is selected for the corresponding dual variable to be updated. If the sample-client pair is not selected to be updated, then $\Delta \alpha_{b,i}^t$ is zero, and the entire right hand size of the inequality is zero. There are $HQ/C$ possible updates in a given outer iteration and $NQ/C$ total sample-client combinations. Therefore, the probability of a given sample-client pair to be selected to be updated is $H/N$. By conditioning on $\alpha_i^{t-1}$, we get

\begin{align}
    \mathbb{E}[A-B|\alpha^{t-1}] &\geq -\frac{H}{QN/C} \sum_{i=1}^N \sum_{b\in\mathcal{B}_c} [ l_i^*(-\alpha_i^{t-1}-\gamma_ts_{b,i}^t(u_i^{t-1}-\alpha_i^{t-1})) - l_i^*(-\alpha_i^{t-1})+\gamma_t(\textbf{A}\alpha^{t-1})^T(s_{b,i}^t(u_i^{t-1}-\alpha_i^{t-1})) x_i\nonumber\\ &+\frac{1}{2\lambda} (\gamma_ts_{b,i}^t(u_i^{t-1}-\alpha_i^{t-1}))^2] \nonumber\\
    &=-\frac{H}{QN/C} \sum_{i=1}^N \sum_{b\in\mathcal{B}_c} [ l_i^*(-\alpha_i^{t-1}-\gamma_ts_{b,i}^t(u_i^{t-1}-\alpha_i^{t-1})) - l_i^*(-\alpha_i^{t-1})+\gamma_ts_{b,i}^t(\textbf{A}\alpha^{t-1})^T(u_i^{t-1}-\alpha_i^{t-1}) x_i\nonumber\\ &+\frac{1}{2\lambda} (\gamma_ts_{b,i}^t(u_i^{t-1}-\alpha_i^{t-1}))^2].
    \label{vert_comp}
\end{align}

We now express $\textbf{A}\alpha^{t-1}$ in terms of $w^{t-1}$. In the vertical case with incomplete client participation, $w^{t-1}$ is different because only a subset of clients are available at each outer iteration. Without loss of generality, let us assume that at iteration $t-1$ set 1 was updated, at iteration $t-2$ set 2 was updated, etc. We define $x_{c,i}$ as the portion of features that are available to clients in set $c$ with padded zeros at indices that are not available to the clients in set $c$.

Then based on PrimalAggregation we have
\begin{align*}
    w^{t-1} &= \frac{1}{\lambda N}\sum_{i=1}^N \alpha_i^{t-1}x_{1,i} + \frac{1}{\lambda N}\sum_{i=1}^N \alpha_i^{t-2}x_{2,i} + \dots + \frac{1}{\lambda N} \sum_{i=1}^N \alpha_i^{t-C}x_{C,i}\\
    &= \textbf{A}\alpha^{t-1} - \frac{1}{\lambda N}\sum_{c=2}^C\sum_{i=1}^N \alpha_i^{t-1}x_{c,i} + \frac{1}{\lambda N}\sum_{c=2}^C\sum_{i=1}^N \alpha_i^{t-c}x_{c,i}.
\end{align*}

We next manipulate this equation to
\begin{align*}
    \textbf{A}\alpha^{t-1} &= w^{t-1} + \frac{1}{\lambda N}\sum_{c=2}^C\sum_{i=1}^N \alpha_i^{t-1}x_{c,i} - \frac{1}{\lambda N}\sum_{c=2}^C\sum_{i=1}^N \alpha_i^{t-c}x_{c,i}\\
    &= w^{t-1} + \frac{1}{\lambda N}\sum_{c=2}^C\sum_{i=1}^N (\alpha_i^{t-1} -\alpha_i^{t-c})x_{c,i}\\
    &= w^{t-1} + \frac{1}{\lambda N}\sum_{c=2}^C\sum_{i=1}^N \sum_{d=1}^{c-1} \Delta \alpha_i^{t-d}x_{c,i}
\end{align*}

and substitute this into (\ref{vert_comp}) to obtain
\begin{align*}
    \mathbb{E}[A-B|\alpha^{t-1}] &\geq -\frac{H}{QN/C} \sum_{i=1}^N \sum_{b\in\mathcal{B}_c} [ l_i^*(-\alpha_i^{t-1}-\gamma_ts_{b,i}^t(u_i^{t-1}-\alpha_i^{t-1})) - l_i^*(-\alpha_i^{t-1})+\gamma_ts_{b,i}^t(w^{t-1})^T(u_i^{t-1}-\alpha_i^{t-1}) x_i\\ &+\gamma_ts_{b,i}^t(\frac{1}{\lambda N}\sum_{c=2}^C\sum_{j=1}^N \sum_{d=1}^{c-1} \Delta \alpha_j^{t-d}x_{c,j})^T(u_i^{t-1}-\alpha_i^{t-1}) x_i +\frac{1}{2\lambda} (\gamma_ts_{b,i}^t(u_i^{t-1}-\alpha_i^{t-1}))^2].\\
\end{align*}

Based on LocalDualMethod
\begin{align*}
    s_{b,i} = \argmax_{s\in[0,1]} -l_i^*(-\alpha_i^{t-1}-\gamma_ts(u_i^{t-1}-\alpha_i^{t-1}))-\gamma_ts(u_i^{t-1}-\alpha_i^{t-1})(w^{t-1})^Tx_i-\frac{\gamma_t^2s^2}{2\lambda}(u_i^{t-1}-\alpha_i^{t-1})^2,
\end{align*}
therefore, we can replace $s_{b,i}$ in the terms included in the maximization with any value in $[0,1]$ and the bound still holds. Furthermore, we know that $s_{b,i}$ are the same for $b \in \mathcal{B}_c$. Therefore, we simply set $s_{b,i}=1$ in the terms in maximization and use convexity of $l_i^*$ to get

\begin{align}
    \mathbb{E}[A-B|\alpha^{t-1}] &\geq -\frac{H}{N}\sum_{i=1}^N [ l_i^*(-\alpha_i^{t-1}-\gamma_t(u_i^{t-1}-\alpha_i^{t-1})) - l_i^*(-\alpha_i^{t-1})+\gamma_t(w^{t-1})^T(u_i^{t-1}-\alpha_i^{t-1}) x_i\nonumber\\ &+\gamma_ts_{b,i}^t(\frac{1}{\lambda N}\sum_{c=2}^C\sum_{j=1}^N \sum_{d=1}^{c-1} \Delta \alpha_j^{t-d}x_{c,j})^T(u_i^{t-1}-\alpha_i^{t-1}) x_i +\frac{1}{2\lambda} (\gamma_t(u_i^{t-1}-\alpha_i^{t-1}))^2]\nonumber\\
    &\geq -\frac{H}{N}\sum_{i=1}^N [ -\gamma_tl_i^*(-\alpha_i^{t-1}) +\gamma_tl_i^*(-u_i^{t-1})+\gamma_t(w^{t-1})^T((u_i^{t-1}-\alpha_i^{t-1})) x_i\nonumber\\ &+\gamma_ts_{b,i}^t(\frac{1}{\lambda N}\sum_{c=2}^C\sum_{j=1}^N \sum_{d=1}^{c-1} \Delta \alpha_j^{t-d}x_{c,j})^T(u_i^{t-1}-\alpha_i^{t-1}) x_i +\frac{1}{2\lambda} (\gamma_t(u_i^{t-1}-\alpha_i^{t-1}))^2].
    \label{vert_comp2}
\end{align}

By using the fact that $s_{b,i}^t \in [0,1]$, we bound the next to last term as
\begin{align*}
    \gamma_t s_{b,i}^t(\frac{1}{\lambda N}\sum_{c=2}^C\sum_{j=1}^N \sum_{d=1}^{c-1} \Delta \alpha_j^{t-d}x_{c,j})^T(u_i^{t-1}-\alpha_i^{t-1})x_i &= \gamma_ts_{b,i}^t(u_i^{t-1}-\alpha_i^{t-1}) (\frac{1}{\lambda N}\sum_{c=2}^C\sum_{j=1}^N \sum_{d=1}^{c-1} \Delta \alpha_j^{t-d}x_{c,j}^Tx_i)\\
    &\leq |\gamma_ts_{b,i}^t(u_i^{t-1}-\alpha_i^{t-1}) (\frac{1}{\lambda N}\sum_{c=2}^C\sum_{j=1}^N \sum_{d=1}^{c-1} \Delta \alpha_j^{t-d}x_{c,j}^Tx_i)|\\
    &= \frac{\gamma_ts_{b,i}^t}{\lambda N}|(u_i^{t-1}-\alpha_i^{t-1})| |(\sum_{c=2}^C\sum_{j=1}^N \sum_{d=1}^{c-1} \Delta \alpha_j^{t-d}x_{c,j}^Tx_i)|\\
    &\leq \frac{\gamma_t|(u_i^{t-1}-\alpha_i^{t-1})|}{\lambda N} \sum_{c=2}^C\sum_{j=1}^N \sum_{d=1}^{c-1}|\Delta \alpha_j^{t-d}x_{c,j}^Tx_i|\\
    &= \frac{\gamma_t|(u_i^{t-1}-\alpha_i^{t-1})|}{\lambda N} \sum_{c=2}^C\sum_{j=1}^N \sum_{d=1}^{c-1}|\Delta \alpha_j^{t-d}||x_{c,j}^Tx_i|\\
    &\leq \frac{\gamma_t|(u_i^{t-1}-\alpha_i^{t-1})|}{\lambda N} \sum_{c=2}^C\sum_{j=1}^N \sum_{d=1}^{c-1}|\Delta \alpha_j^{t-d}|\\
    &= \frac{\gamma_t|(u_i^{t-1}-\alpha_i^{t-1})|}{\lambda N} \sum_{c=2}^C\sum_{j=1}^N \sum_{d=1}^{c-1}|\frac{\gamma_{t-d}}{Q/C}\sum_{b\in\mathcal{B}_c}\Delta \alpha_{b,j}^{t-d}|.
\end{align*}
By definition, $\Delta\alpha_{b,j}^{t-d}=s_{b,j}^t(u_j^{t-d-1}-\alpha_j^{t-d-1})$ and thus

\begin{align*}
    \gamma_t (\frac{1}{\lambda N}\sum_{c=2}^C\sum_{i=1}^N \sum_{d=1}^{c-1} \Delta \alpha_i^{t-d}x_{c,i})^T(u_i^{t-1}-\alpha_i^{t-1})x_i 
    &\leq \frac{\gamma_t|(u_i^{t-1}-\alpha_i^{t-1})|}{\lambda N} \sum_{c=2}^C\sum_{j=1}^N \sum_{d=1}^{c-1}\gamma_{t-d}|u_j^{t-d-1}-\alpha_j^{t-d-1}|\cdot \max_{b\in\mathcal{B}_c}|s_{b,j}^t|\\
    &\leq \frac{2L\gamma_t}{\lambda N} \sum_{c=2}^C\sum_{j=1}^N \sum_{d=1}^{c-1}2L\gamma_{t-d}\\
    &\leq \frac{4L^2\gamma_t}{\lambda N} \sum_{c=2}^C\sum_{j=1}^N \sum_{d=1}^{c-1}\gamma_{t-d}\\
    &\leq \frac{4L^2\gamma_t}{\lambda} \sum_{c=2}^C (c-1)\gamma_{t-c+1}\\
    &\leq \frac{4L^2\gamma_t(C-1)^2}{\lambda} \gamma_{t-C+1}\\
    &\leq \frac{4L^2\gamma_{t-C+1}^2(C-1)^2}{\lambda}.\\
\end{align*}

We continue with (\ref{vert_comp2}) to derive
\begin{align*}
    \mathbb{E}[A-B|\alpha^{t-1}]
    &\geq -\frac{H}{N}\sum_{i=1}^N [ -\gamma_t l_i^*(-\alpha_i^{t-1}) + \gamma_t l_i^*(-u_i^{t-1})+\gamma_t(w^{t-1})^T(u_i^{t-1}-\alpha_i^{t-1}) x_i \\ &+ \frac{4L^2\gamma_{t-C+1}^2(C-1)^2}{\lambda}+\frac{1}{2\lambda} (\gamma_t(u_i^{t-1}-\alpha_i^{t-1}))^2].\\
\end{align*}

Now from convex conjugates we know that $l_i(x_i^Tw^{t-1})=-l_i^*(-u_i^{t-1})-u_i^{t-1}x_i^Tw^{t-1}$ which yields
\begin{align}
    \mathbb{E}[A-B|\alpha^{t-1}]
    &\geq -\frac{H}{N}\sum_{i=1}^N [ -\gamma_t l_i^*(-\alpha_i^{t-1}) - \gamma_t l_i(x_i^Tw^{t-1})-\gamma_t(w^{t-1})^T\alpha_i^{t-1} x_i \nonumber\\ &+ \frac{4L^2\gamma_{t-C+1}^2(C-1)^2}{\lambda}+\frac{1}{2\lambda} (\gamma_t(u_i^{t-1}-\alpha_i^{t-1}))^2].\\
    \label{vert_comp3}
\end{align}

We next explore the primal-dual gap. We set $D=\frac{1}{\lambda N}\sum_{c=2}^C\sum_{i=1}^N \sum_{d=1}^{c-1} \Delta \alpha_i^{t-d}x_{c,i}$ to obtain
\begin{align*}
    P(w^{t-1})-D(\alpha^{t-1}) &= (\frac{\lambda}{2}||w^{t-1}||^2 + \frac{1}{N}\sum_{i=1}^Nl_i(x_i^Tw^{t-1}))-(-\frac{\lambda}{2}|| \textbf{A}\alpha^{t-1}||^2-\frac{1}{N}l_i^*(-\alpha_i^{t-1}))\\
    &= (\frac{\lambda}{2}||w^{t-1}||^2 + \frac{1}{N}\sum_{i=1}^Nl_i(x_i^Tw^{t-1}))-(-\frac{\lambda}{2}||w^{t-1}+D||^2-\frac{1}{N}l_i^*(-\alpha_i^{t-1})).\\
\end{align*}

We have that $\lambda ||w^{t-1}||^2 = \lambda(w^{t-1})^T w^{t-1} = \lambda (w^{t-1})^T(\textbf{A}\alpha^{t-1}-D)=\lambda (w^{t-1})^T\textbf{A}\alpha^{t-1}-\lambda (w^{t-1})^TD$.

This results in
\begin{align*}
    P(w^{t-1})-D(\alpha^{t-1}) &= \frac{1}{N}\sum_{i=1}^N[l_i(x_i^Tw^{t-1})+l_i^*(-\alpha_i^{t-1})] + \lambda (w^{t-1})^T\textbf{A}\alpha^{t-1} + \frac{\lambda}{2}||D||^2\\
    &= \frac{1}{N}\sum_{i=1}^N[l_i(x_i^Tw^{t-1})+l_i^*(-\alpha_i^{t-1}) + (w^{t-1})^T\alpha_i^{t-1}x_i] + \frac{\lambda}{2}||D||^2\\
\end{align*}

and in turn
\begin{align*}
    P(w^{t-1})-D(\alpha^{t-1}) - \frac{\lambda}{2}||D||^2 
    &= \frac{1}{N}\sum_{i=1}^N[l_i(x_i^Tw^{t-1})+l_i^*(-\alpha_i^{t-1}) + (w^{t-1})^T\alpha_i^{t-1}x_i]. \\
\end{align*}

We substitute this into (\ref{vert_comp3}) to get
\begin{align}
    \mathbb{E}[A-B|\alpha^{t-1}]
    &\geq H\gamma_t[P(w^{t-1})-D(\alpha^{t-1})] -\frac{H\gamma_t \lambda}{2}||D||^2 -\frac{H}{N}\sum_{i=1}^N [\frac{4L^2\gamma_{t-C+1}^2(C-1)^2}{\lambda}+\frac{1}{2\lambda} (\gamma_t(u_i^{t-1}-\alpha_i^{t-1}))^2].
    \label{vert_comp4}
\end{align}

To bound $||D||^2$ we use
\begin{align*}
    ||D||^2 &= ||\frac{1}{\lambda N}\sum_{c=2}^C\sum_{i=1}^N \sum_{d=1}^{c-1} \Delta \alpha_i^{t-d}x_{c,i}||^2\\
    &= \frac{1}{\lambda^2 N^2}||\sum_{c=2}^C\sum_{i=1}^N \sum_{d=1}^{c-1} \Delta \alpha_i^{t-d}x_{c,i}||^2\\
    &\leq \frac{(C-1)^2}{\lambda^2 N}\sum_{c=2}^C\sum_{i=1}^N \sum_{d=1}^{c-1}(\Delta \alpha_i^{t-d})^2||x_{c,i}||^2\\
    &\leq \frac{(C-1)^2}{\lambda^2 N}\sum_{c=2}^C\sum_{i=1}^N \sum_{d=1}^{c-1}(\Delta \alpha_i^{t-d})^2\\
    &= \frac{(C-1)^2}{\lambda^2 N}\sum_{c=2}^C\sum_{i=1}^N \sum_{d=1}^{c-1}(\frac{\gamma_{t-d}}{Q/C}\sum_{b\in\mathcal{B}_c}\Delta \alpha_{b,i}^{t-d})^2\\
    &= \frac{(C-1)^2}{\lambda^2 N}\sum_{c=2}^C\sum_{i=1}^N \sum_{d=1}^{c-1}(\gamma_{t-d}\frac{1}{Q/C}\sum_{b\in\mathcal{B}_c}s_{b,i}^t (u_i^{t-1}-\alpha_i^{t-1}))^2\\
    &\leq \frac{(C-1)^2}{\lambda^2 N}\sum_{c=2}^C\sum_{i=1}^N \sum_{d=1}^{c-1}\gamma_{t-d}^2 (u_i^{t-1}-\alpha_i^{t-1})^2\cdot (\frac{1}{Q/C}\sum_{b\in\mathcal{B}_c}s_{b,i})^2\\
    &\leq \frac{4L^2(C-1)^2}{\lambda^2 N}\sum_{c=2}^C\sum_{i=1}^N \sum_{d=1}^{c-1}\gamma_{t-d}^2\\
    &\leq \frac{4L^2(C-1)^2}{\lambda^2 N}\sum_{c=2}^C\sum_{i=1}^N (c-1)\gamma_{t-c+1}^2\\
    &\leq \frac{4L^2(C-1)^4}{\lambda^2}\gamma_{t-C+1}^2.\\
\end{align*}

We substitute this into (\ref{vert_comp4}) to get

\begin{align*}
    \mathbb{E}[A-B|\alpha^{t-1}]
    &\geq H\gamma_t[P(w^{t-1})-D(\alpha^{t-1})] -\frac{H\gamma_t \lambda}{2}\frac{4L^2(C-1)^4}{\lambda^2}\gamma_{t-C+1}^2 -\frac{H}{N}\sum_{i=1}^N [\frac{4L^2\gamma_{t-C+1}^2(C-1)^2}{\lambda}\\&+\frac{1}{2\lambda} (\gamma_t(u_i^{t-1}-\alpha_i^{t-1}))^2]\\
    &\geq H\gamma_t[P(w^{t-1})-D(\alpha^{t-1})] -\frac{2L^2H(C-1)^4}{\lambda}\gamma_{t-C+1}^2 -\frac{H}{N}\sum_{i=1}^N [\frac{4L^2\gamma_{t-C+1}^2(C-1)^2}{\lambda}+\frac{2L^2\gamma_t^2}{\lambda}]\\
    &\geq H\gamma_t[P(w^{t-1})-D(\alpha^{t-1})] -\frac{2L^2H(C-1)^4}{\lambda}\gamma_{t-C+1}^2 -H[\frac{4L^2\gamma_{t-C+1}^2(C-1)^2}{\lambda}+\frac{2L^2\gamma_{t-C+1}^2}{\lambda}]\\
    &\geq H\gamma_t[P(w^{t-1})-D(\alpha^{t-1})] -\gamma_{t-C+1}^2H\underbrace{[\frac{2L^2(C-1)^4}{\lambda} +\frac{4L^2(C-1)^2}{\lambda}+\frac{2L^2}{\lambda}]}_{G}.
\end{align*}

As in the previous proofs this yields
\begin{align*}
    \mathbb{E}[\varepsilon_D^t] &\leq (1-\frac{\gamma_t H}{N}) \mathbb{E}[\varepsilon_D^{t-1}] +\frac{\gamma_{t-C+1}^2 H}{N}G
\end{align*}

and in turn
\begin{align*}
    \mathbb{E}[\varepsilon_D^t] &\leq (1-\frac{\gamma_t H}{N})[(1-\frac{\gamma_{t-1} H}{N}) \mathbb{E}[\varepsilon_D^{t-2}] +\frac{\gamma_{t-C}^2 H}{N}G] +\frac{\gamma_{t-C+1}^2 H}{N}G\\
    &\leq (1-\frac{\gamma_t H}{N})[(1-\frac{\gamma_{t-1} H}{N}) [(1-\frac{\gamma_{t-2} H}{N}) \mathbb{E}[\varepsilon_D^{t-3}] +\frac{\gamma_{t-C-1}^2 H}{N}G] +\frac{\gamma_{t-C}^2 H}{N}G] +\frac{\gamma_{t-C+1}^2 H}{N}G\\
    &\leq \underbrace{\prod_{\tau=C}^t(1-\frac{H\gamma_\tau}{N})\mathbb{E}[\varepsilon_D^{C-1}]}_{B_1} + \underbrace{\frac{GH}{N}\sum_{i=1}^{t-C+1}\gamma_i^2 \prod_{\tau=i}^{t-C}(1-\frac{H\gamma_{\tau+C}}{N})}_{B_2}.
\end{align*}
Note that $\frac{H}{N}\leq 1$ by definition.

First, we bound $B_1$. Using the fact that $\ln(1-x)\leq -x$ for every $0\leq x \leq 1$ and that $\frac{1}{\tau}$ monotonically decreases for $\tau>0$, we have
\begin{align*}
    \ln(\prod_{\tau=C}^t(1-\frac{H}{N\tau})) &= \sum_{\tau=C}^t\ln(1-\frac{H}{\tau N})\\
    &\leq -\frac{H}{N}\sum_{\tau=C}^t \frac{1}{\tau}\\
    &\leq -\frac{H}{N}\int_{C}^{t+1} \frac{1}{\tau'}d\tau'\\
    &= -\frac{H}{N}\ln(\tau ')\bigg|_{C}^{t+1}\\
    &= -\frac{H}{N}\ln(\frac{t+1}{C})\\
    &= \ln((\frac{t+1}{C})^{-\frac{H}{N}}).\\
\end{align*}
Thus we have
\begin{align*}
    B_1 &\leq \mathbb{E}[\varepsilon_D^{C-1}]\frac{C^{H/N}}{(t+1)^{H/N}}.\\
\end{align*}

Next, we bound $B_2$. Using the fact that $\ln(1-x)\leq -x$ for every $0\leq x \leq 1$ and that $\frac{1}{\tau+C}$ monotonically decreases for $\tau>0$ we have
\begin{align*}
    \ln(\prod_{\tau=i}^{t-C}(1-\frac{H\gamma_{\tau+C}}{N}))&= \sum_{\tau=i}^{t-C} \ln(1-\frac{H\gamma_{\tau+C}}{N})\\
    &\leq -\frac{H}{N}\sum_{\tau=i}^{t-C}\frac{1}{\tau+C}\\
    &\leq -\frac{H}{N} \int_{i}^{t-C+1}\frac{1}{\tau'+C}d\tau'\\
    &=-\frac{H}{N} \ln(\tau'+C) \bigg |^{t-C+1}_{i}\\
    &=-\frac{H}{N}\ln(\frac{t+1}{i+C})\\
    &=\ln((\frac{i+C}{t+1})^\frac{H}{N}).
\end{align*}

Therefore, we have $\prod_{\tau=i}^{t-C}(1-\frac{H\gamma_{\tau+C}}{N}) \leq (\frac{i+C}{t+1})^\frac{H}{N}$.

We know that $H/N\leq1$ and thus
\begin{align*}
    B_2 &= \frac{GH}{N}\sum_{i=1}^{t-C+1}\gamma_i^2 \prod_{\tau=i}^{t-C}(1-\frac{H\gamma_{\tau+C}}{N})\\
    &\leq \frac{GH}{N}\sum_{i=1}^{t-C+1}\gamma_i^2 (\frac{i+C}{t+1})^\frac{H}{N}\\
    &\leq \frac{GH}{N}\sum_{i=1}^{t-C+1}\gamma_i^2 (\frac{i+C}{t})^\frac{H}{N}\\
    &= \frac{GH}{N}t^{-H/N}\sum_{i=1}^{t-C+1}\frac{(i+C)^\frac{H}{N}}{i^2}\\
    &\leq \frac{GH}{N}t^{-H/N}\sum_{i=1}^{t-C+1}\frac{i+C}{i^2}\\
    &= \frac{GH}{N}t^{-H/N}\sum_{i=1}^{t-C+1}[\frac{1}{i} + \frac{C}{i^2}]\\
    &\leq \frac{GH}{N}t^{-H/N}\sum_{i=1}^{t-C+1}[\frac{1}{i} + \frac{C}{i}]\\
    &= \frac{GH}{N}t^{-H/N}(C+1)\sum_{i=1}^{t-C+1}\frac{1}{i}\\
    &= \frac{GH}{N}t^{-H/N}(C+1)\text{Har}_{t-C+1},\\
\end{align*}
where $\text{Har}_{t-C+2}$ is the $(t-C+1)$-th harmonic number. Using the well-known bound on the harmonic numbers $\text{Har}_n\leq \ln(n)+1$, we have
\begin{align*}
    B_2 \leq \frac{GH(C+1)(\ln(t-C+1)+1)}{Nt^{H/N}}.\\
\end{align*}

All that is left is to bound $\mathbb{E}[\varepsilon_D^{C-1}]$, which follows a similar recurrence relation as before but with slightly different terms. This results in a finite expression that depends on $\varepsilon_D^0$. Since $\varepsilon_D^0 \leq 1$ (proved in Lemma 20 of \citep{shalev}), it is clear that $\mathbb{E}[\varepsilon_D^{C-1}]$ is bounded by a constant.

Thus we have
\begin{align*}
    \mathbb{E}[\varepsilon_D^{t}] \leq \frac{J_1 + J_2(\ln(t-C+1)+1)}{t^{H/N}}
\end{align*}
where
\begin{align*}
    J_1 &= C^{H/N}\mathbb{E}[\varepsilon_D^{C-1}]\\
    J_2 &= \frac{2HL^2(C+1)}{N\lambda}[(C-1)^4  +2(C-1)^2+1].
\end{align*}

This completes the proof.

\end{proof}

\clearpage

\section{Computational Study}
\label{app_b}
We provide the full details of the experiments here. The details provided are enough to reproduce results, and the full set of experimental results are provided here as well.

\subsection{Implementation}

In order to investigate the performance of HyFDCA, we implemented HyFDCA using the practical LocalDualMethod as described in Algorithm 5 as well as FedAvg and HyFEM using Python 3.10. NumPy was used to handle all matrices and matrix algebra for the MNIST and Covtype datasets, and due to the sparsity of the News20 dataset, we employed SciPy's sparse matrices. 

Kubernetes was employed for container management to enable several concurrent experiments. Each dataset's comparable experiments were run on identical pods on the same node (server) for all three algorithms to ensure fair comparison of results. We note that, experiments terminated based on a preset number of outer iterations ("Equal Outer Iterations") and those terminated based on a preset wall time ("Equal Wall Time") were run on different hardware. Since no direct comparisons were made between these data and all comparable results were run using identical hardware, we emphasize these are fair and comparable experiments:

\begin{itemize}
    \item \textbf{MNIST Equal Outer Iterations:} 12 Intel\textregistered\: Core\texttrademark\: i7-6850K CPU @ 3.60GHz, 6 cores each
    \item \textbf{MNIST Equal Wall Time:} 32 Intel\textregistered\: Xeon\textregistered\: Silver 4208 CPU @ 2.10GHz, 8 cores each
    \item \textbf{News20 Equal Outer Iterations:} 16 Intel\textregistered\: Core\texttrademark\: i7-7820X CPU @ 3.60GHz, 8 cores each 
    \item \textbf{News20 Equal Wall Time:} 48 Intel\textregistered\: Xeon\textregistered\: Silver 4214 CPU @ 2.20GHz, 12 cores each
    \item \textbf{Covtype Equal Outer Iterations:} 16 Intel\textregistered\: Core\texttrademark\: i7-7820X CPU @ 3.60GHz, 8 cores each 
    \item \textbf{Covtype Equal Wall Time:} 16 Intel\textregistered\: Core\texttrademark\: i7-7820X CPU @ 3.60GHz, 8 cores each

\end{itemize}

In FL each client would perform local computations in parallel, but in our simulation the client objects were updated sequentially and the slowest client was used to record the time (to accurately simulate the clients working in parallel).

Finally, homomorphic encryption was not actually conducted as part of the experiment. Our simulation simply found the number of necessary encryptions/operations in each iteration and used this to compute the estimated encryption time penalty based on published benchmarks for homomorphic encryption algorithms. For the first step of the algorithm, there is a decryption penalty only for the dual variables that were updated and are stale on the clients. This variable depends on the iteration. For the SecureInnerProduct, only the sample indices that are going to be chosen to be updated on at least one of the clients need to be encrypted and sent to the server for addition and back for decryption. Since there may be significant overlap in the indices chosen between clients, this is also variable depending on the iteration. The dual variable updates found on each client must also be encrypted before sending to the server. The number of dual variables that then need to be decrypted by each client varies because the overlap in which dual variables are selected for updates. In the experiments, the number of necessary encryption/decryption operations were found in each iteration and used to calculate the encryption penalties to be added to the total wall time.

The total number of outer iterations used for each problem setting was 2,500 outer iterations for MNIST, 10,000 outer iterations for News20, and 30,000 outer iterations for Covtype. For equal wall time experiments, 3 hours (10,800 seconds) of computation as if computed in parallel was completed.


\subsection{Data Partitioning}
Table \ref{dataset-info-table} outlines the key characteristics of each dataset. Sparsity is defined as the percent of zero values divided by the total number of values.

\begin{table}[h]
\caption{Dataset Information} \label{dataset-info-table}
\begin{center}
\begin{tabular}{r|lll}
& \textbf{MNIST} & \textbf{News20} & \textbf{Covtype} \\
\hline \\
\textbf{Type} & Image & Text & Multivariate \\
\textbf{Classes} & 10 & 2 & 2 \\
\textbf{Samples} & 70,000 & 19,996 & 581,012 \\
\textbf{Features} & 784 & 1,355,191 & 54 \\
\textbf{Sparsity} & $80.858\%$ & $99.966\%$ & $77.878\%$ \\
\end{tabular}
\end{center}
\end{table}

Due to the differences in the datasets, specifically the meaning of the features, different methodologies were used to partition the data among a number of clients while ensuring that approximately the same number of non-zero data-points were included for each client. We note, however, that data is inherently non-IID in the hybrid FL case because each client stores different sections of the feature space.

Each MNIST sample is a $28$x$28$ pixel image where values were normalized to $[0,1]$. Each sample was split into four equal-sized quadrants, and thus $Q=4$ for all problem settings. A bias feature of value 10 was appended to the fourth quadrant's data; this larger value was chosen to prevent the bias term from being affected as strongly by regularization. Each client was then provided with $\frac{4N}{KQ}$ sample quadrants where $KQ$ is the total number of clients. The first quarter of clients received features from the first quadrant of the images, the second quarter of clients received features from the second quadrant of the images, and so on.

News20 and Covtype used a different assignment procedure as geometrically segmenting features makes no sense for either dataset. First, the samples are evenly divided into $K$ number of sample groups. Within each sample group, Q clients are defined. Examining each sample within a segment separately, each client receives $1/Q$ of the non-zero features. No bias term was employed for either the News20 or Covtype datasets. As such, there were $K\cdot Q$ clients. Similar to MNIST, this process ensures that there is no data overlap between clients and minimizes data imbalances. 

\subsection{Hyperparameter Tuning}

For HyFDCA, the only hyperparameter to tune was the number of inner iterations because no diminishing learning rate is needed for convergence. FedAvg requires tuning of the number of inner iterations on each client and the constants $a$ and $b$ in the learning rate $\gamma_t = \frac{a}{b+\sqrt{t}}$. Inner iterations in FedAvg correspond to the number of times the primal weights are updated in a given iteration. Our implementation used a batch size of one to find stochastic gradients, and made $H$ updates to the primal weights in each iteration. Finally, HyFEM is similar to FedAvg in needing to tune ICC, $a$, $b$ but also requires $\mu$, which balances the two losses.

The number of inner iterations was not directly tuned. Rather, we employed an inner iterations constant, $IIC$, that defined the number of inner iterations as a function of the number of training samples, total number of clients, and $IIC$ as follows:

\begin{equation}
\label{iic_calc}
    H = \ceil[\Bigg]{\frac{IIC\cdot N}{KQ}}.
\end{equation}

We employed a time-bounded random search method in order to collect sufficient data to select hyperparameters \citep{random-search}. The search was run for 5 days (432,000 seconds) per dataset-algorithm combination. Nine client-fraction combinations were tuned independently for each dataset. For MNIST, this consisted of all combinations of $5$, $500$, and $5000$ clients with $0.1$, $0.5$, $0.9$ fraction of clients available. For News20 and Covtype, all combinations using three sets of (sample groups, feature groups) - $(3,3)$, $(5,5)$, $(12,12)$ - were examined using the same three fractions - $0.1$, $0.5$, $0.9$ - leading to a 9 client-fraction combinations. Based on an understanding of reasonable hyperparameters from preliminary testing, the search area was bounded as follows:
\begin{itemize}
    \item HyFDCA IIC: $[\frac{\min{clients}}{samples}, 1.0]$
    \item FedAvg IIC: $[\frac{\min{clients}}{samples}, 5.0]$
    \item FedAvg a: $[10^{-5}, 25.0]$
    \item FedAvg b: $[10^{-5}, 25.0]$.
    \item HyFEM IIC: $[\frac{\min{clients}}{samples}, 5.0]$
    \item HyFEM a: $[10^{-5}, 25.0]$
    \item HyFEM b: $[10^{-5}, 25.0]$.
    \item HyFEM $\mu$: $[10^{-2}, 10^{4}]$.
\end{itemize}

We used (\ref{iic_calc}) to calculate appropriate ranges for $IIC$ for each dataset. From preliminary testing, we believed that smaller values of all hyperparameters were more likely to be chosen as optimal leading us to sample values from a logarithmic distribution. We randomly sampled a value, $x_i$, from the uniform distribution $[log_{10}x_{min}, log_{10}x_{max}]$ where $x_{min}$ and $x_{max}$ are the lower and upper bounds of the given hyperparameter, respectively. The randomly selected hyperparameter is therefore defined as $10^{x_i}$.


Hyperparameter tuning in the case of federated learning is more complicated due to the large number of competing metrics that define an algorithm's performance. For example, we may wish to minimize the total number of outer iterations (communication rounds) to reach a satisfactory loss function value but also wish to minimize the total computation time on each client because of computational limits on devices such as smart phones. These two goals are directly conflicting. For this reason, we frame this hyperparameter selection problem with multiple metrics as a multiobjective optimization problem where the optimal solution must be selected from the Pareto-Optimal front. We solve this using Gray Relational Analysis as described in \citet{GRA}. The metrics we use for the hyperparameter selection are as follows:
\begin{enumerate}
    \item Average runtime per iteration with $0.000$ seconds of round-trip latency
    \item Average runtime per iteration with $0.2575$ seconds of round-trip latency
    \item Average runtime per iteration with $0.8000$ seconds of round-trip latency
    \item Average of last 5 loss function values
    \item Maximum validation accuracy
    \item Volatility - standard deviation of differences in consecutive loss function values
    \item Number of iterations to reach $90\%$ progress of minimizing the loss function
\end{enumerate}
Here, $0.000$ seconds of round-trip latency is meant to represent the most ideal scenario, $0.2575$ seconds represents a long-distance server connection (US-Singapore AWS server representation \citep{AWS-latency}), and $0.8000$ represents a GEO satellite connection \citep{GEO-latency}. There are $4.5$ round-trips of information transmissions for HyFDCA per iteration, $1.0$ round-trips for FedAvg per iteration, and $1.0$ round-trips for HyFEM per iteration.

All divergent runs are excluded from GRA. All optimally selected hyperparameters for HyFDCA, FedAvg, and HyFEM are listed in Tables \ref{MNIST-hyperparameters}-\ref{Covtype-hyperparameters}. We note, however, that we are unable to find a viable set of hyperparameters for HyFEM in the News20 dataset despite doubling the search time to 10 days (864,000 seconds). Therefore, we use FedAvg's hyperparameters with $\mu = 1$ as was done in the original HyFEM paper \citep{hyfem}. This demonstrates a key issue with FedAvg and HyFEM that HyFDCA alleviates:  FedAvg does not guarantee convergence and HyFEM does not guarantee convergence to the centralized solution while both are highly sensitive to the choice of hyperparameters. In other words, if hyperparameters are not well tuned, either algorithm simply will not make progress and FedAvg could completely diverge. In addition to being highly sensitive, FedAvg's and HyFEM's hyperparameters are hard to tune with the viable hyperparameter range appearing to be relatively narrow. Though we tune HyFDCA's IIC parameter for fair comparisons, tuning is not required for convergence and only helps increase convergence efficiency.

\begin{table*}[h]
\small
\caption{MNIST optimal hyperparameters employed} 
\label{mnist-hyperparameter-table}
\begin{center}
\begin{tabular}{r|llllllll}
& \textbf{HyFDCA IIC} & \textbf{FedAvg IIC} & \textbf{FedAvg a} & \textbf{FedAvg b} & \textbf{HyFEM IIC} & \textbf{HyFEM a} & \textbf{HyFEM b} & \textbf{HyFEM $\mu$} \\
\hline
\textbf{5, 0.1} & 0.0001876 & 0.04872 & 0.05857 & 0.1124 & 0.00568 & 0.00006007 & 1.015 & 6.296 \\
\textbf{5, 0.5} & 0.001564 & 0.04872 & 0.05857 & 0.1124 & 0.00568 & 0.00006007 & 1.015 & 6.296 \\
\textbf{5, 0.9} & 0.002816 & 0.01718 & 0.02867 & 0.0004823 & 0.00568 & 0.00006007 & 1.015 & 6.296 \\
\textbf{500, 0.1} & 0.03827 & 0.01718 & 0.02867 & 0.0004823 & 0.05594 & 0.0004672 & 0.05344 & 2.797 \\
\textbf{500, 0.5} & 0.002252 & 0.06561 & 0.04185 & 0.00004166 & 0.00568 & 0.00006007 & 1.015 & 6.296 \\
\textbf{500, 0.9} & 0.08337 & 0.04872 & 0.05857 & 0.1124 & 0.00568 & 0.00006007 & 1.015 & 6.296 \\
\textbf{5000, 0.1} & 0.002673 & 1.618 & 0.003995 & 0.8498 & 0.05594 & 0.0004672 & 0.05344 & 2.797 \\
\textbf{5000, 0.5} & 0.0003154 & 1.618 & 0.003995 & 0.8498 & 0.00568 & 0.00006007 & 1.015 & 6.296 \\
\textbf{5000, 0.9} & 0.02195 & 1.618 & 0.003995 & 0.8498 & 0.05594 & 0.0004672 & 0.05344 & 2.797 \\
\label{MNIST-hyperparameters}
\end{tabular}
\end{center}
\end{table*}

\begin{table*}[h]
\small
\caption{News20 optimal hyperparameters employed}
\label{mnist-hyperparameter-table}
\begin{center}
\begin{tabular}{r|llllllll}
& \textbf{HyFDCA IIC} & \textbf{FedAvg IIC} & \textbf{FedAvg a} & \textbf{FedAvg b} & \textbf{HyFEM IIC} & \textbf{HyFEM a} & \textbf{HyFEM b} & \textbf{HyFEM $\mu$} \\
\hline
\textbf{3, 3, 0.1} & 0.01448 & 0.003037 & 0.8753 & 0.4597 & 0.003037 & 0.8753 & 0.4597 & 1.000 \\
\textbf{3, 3, 0.5} & 0.01187 & 0.002873 & 4.891 & 4.496 & 0.002873 & 4.891 & 4.496 & 1.000 \\
\textbf{3, 3, 0.9} & 0.01122 & 0.002873 & 4.891 & 4.496 & 0.002873 & 4.891 & 4.496 & 1.000 \\
\textbf{5, 5, 0.1} & 0.03072 & 0.002873 & 4.891 & 4.496 & 0.002873 & 4.891 & 4.496 & 1.000 \\
\textbf{5, 5, 0.5} & 0.0286 & 0.003037 & 0.8753 & 0.4597 & 0.003037 & 0.8753 & 0.4597 & 1.000 \\
\textbf{5, 5, 0.9} & 0.03072 & 0.002873 & 4.891 & 4.496 & 0.002873 & 4.891 & 4.496 & 1.000 \\
\textbf{12, 12, 0.1} & 0.08568 & 4.542 & 3.792 & 0.1815 & 0.03000\tablefootnote{In FedAvg, we found this IIC value to be 4.542 but further experiments in HyFEM showed that said set of hyperparameters resulted in divergent behaviour. This further emphasizes the difficulties of selecting hyperparameters for HyFEM in practical use cases. We changed this IIC value to result in convergent behaviour for fair comparisons between algorithms. All data shown was collected using this value.} & 3.792 & 0.1815 & 1.000 \\
\textbf{12, 12, 0.5} & 0.08504 & 4.542 & 3.792 & 0.1815 & 4.542 & 3.792 & 0.1815 & 1.000 \\
\textbf{12, 12, 0.9} & 0.07564 & 4.542 & 3.792 & 0.1815 & 4.542 & 3.792 & 0.1815 & 1.000 \\
\label{News20-hyperparameters}
\end{tabular}
\end{center}
\end{table*}

\begin{table*}[h]
\small
\caption{Covtype optimal hyperparameters employed}
\label{mnist-hyperparameter-table}
\begin{center}
\begin{tabular}{r|llllllll}
& \textbf{HyFDCA IIC} & \textbf{FedAvg IIC} & \textbf{FedAvg a} & \textbf{FedAvg b} & \textbf{HyFEM IIC} & \textbf{HyFEM a} & \textbf{HyFEM b} & \textbf{HyFEM $\mu$} \\
\hline
\textbf{3, 3, 0.1} & 0.0001326 & 0.02291 & 0.4475 & 0.0003501 & 0.03747 & 0.6675 & 0.01209 & 0.01458 \\
\textbf{3, 3, 0.5} & 0.0002643 & 0.02291 & 0.4475 & 0.0003501 & 0.03747 & 0.6675 & 0.01209 & 0.01458 \\
\textbf{3, 3, 0.9} & 0.00005533 & 0.02291 & 0.4475 & 0.0003501 & 0.02038 & 0.4948 & 0.08829 & 0.01155 \\
\textbf{5, 5, 0.1} & 0.00007376 & 0.03647 & 0.3247 & 4.514 & 0.413 & 1.982 & 1.058 & 0.02459 \\
\textbf{5, 5, 0.5} & 0.0008319 & 0.02291 & 0.4475 & 0.0003501 & 0.00004909 & 1.913 & 4.679 & 0.6115 \\
\textbf{5, 5, 0.9} & 0.0001326 & 0.03647 & 0.3247 & 4.514 & 0.05226 & 2.41 & 0.0142 & 0.06684 \\
\textbf{12, 12, 0.1} & 0.0008319 & 0.000559 & 0.05344 & 0.004065 & 0.000102 & 0.1264 & 0.0003874 & 8.775 \\
\textbf{12, 12, 0.5} & 0.0007434 & 0.000559 & 0.05344 & 0.004065 & 0.03747 & 0.6675 & 0.01209 & 0.01458 \\
\textbf{12, 12, 0.9} & 0.0007434 & 0.2494 & 11.93 & 0.05131 & 0.05226 & 2.41 & 0.0142 & 0.06684 \\

\label{Covtype-hyperparameters}
\end{tabular}
\end{center}
\end{table*}

\subsection{Analysis Methods}
Table \ref{centralized-info-table} shows the number of centralized iterations to find the optimal centralized solutions along with the minimum loss value for these centralized runs.

\begin{table}[h]
\caption{Centralized Training Information}
\begin{center}
\begin{tabular}{r|lll}
& \textbf{MNIST} & \textbf{News20} & \textbf{Covtype} \\
\hline \\
\textbf{Iterations} & 30,000 & 40,000 & 1,000,000 \\
\textbf{Min Loss} & 0.07242 & 0.008017 & 0.5286 \\
\label{centralized-info-table}
\end{tabular}
\end{center}
\end{table}

Due to a combination of factors including the large number of outer iterations,  characteristics of the datasets, and the nature of the algorithms employed, there was significant volatility in the loss function and validation accuracy values making plots hard to read. We applied a moving average in order to smooth the loss function and validation accuracy values solely for the sake of readability in plotting. The moving average had a window, $\omega$, of 50 outer iteration for MNIST, 150 for News20, and 300 for Covtype for equal outer iteration experiments as defined above. Since the number of outer iterations for equal wall time experiments ranged between 44 and 11,605,640, we define $\omega$ to be the rounded integer value of 1\% of the number total number of outer iterations where $\omega < 2$ remains not smoothed.

All three algorithms were run for either the same number of outer iterations or same amount of parallel-equivalent wall time without a stopping criterion.

Additionally, we employed a relative time measure and percent of outer iterations completed for clearer and normalized plotting. These were performed separately for each dataset, and it preserves the relative differences between the three algorithms while allowing the results from all three datasets to be plotted together. Such values are defined as follows:

\begin{equation}
    T_R = \frac{T}{\max\{{T_1, T_{2}, T_{3}}\}}
\end{equation}

where $T_R$ is the relative time value, $T_1$ is the total wall time for HyFDCA, $T_2$ is the total wall time for FedAvg, $T_3$ is the total wall time for HyFEM, and T is the original time quantity for either HyFDCA or FedAvg or HyFEM. Additionally, we have

\begin{equation}
    t_R = \frac{t}{t_{max} - \omega}
\end{equation}

where $t_R$ is the percent of outer iterations completed, $\omega$ is the width of the moving average window, $t_{max}$ is the maximum number of outer iterations, and $t$ is the number of outer iterations. 


\subsection{Full Results}

Due to the large number of problem settings investigated and the various metrics of interest, only selected plots were included in the main paper. Additional data are included herein to ensure completeness and transparency in reporting. 

Complementing Figure 3 in the main body, Supplementary Materials Figure \ref{equal_time_fig} compares performance with respect to time including communication latency and homomorphic encryption time. Latency of 0.2575s/RTC is employed as a representative moderate scenario. Here, in 33 of 36 comparisons made in Figure \ref{equal_time_fig}, HyFDCA outperforms FedAvg and HyFEM. The other three scenarios are FedAvg and HyFEM having marginally higher validation accuracy for the case of a low number of clients and high fraction of clients participating for the Covtype dataset as well as FedAvg also having a marginally higher validation accuracy for the case of a medium number of clients and medium fraction of clients participating for the MNIST dataset. In these three scenarios we emphasize that the differences in validation accuracy is exceedingly small. In these three scenarios, we note that HyFDCA may not have fully converged and, if run for a greater wall time, HyFDCA would result in better validation accuracy here as well. This is supported by Figure 3 in the main body where HyFDCA eventually has better validation accuracy during later outer iterations. In all other cases, HyFDCA significantly outperforms both FedAvg and HyFEM in terms of loss function value and validation accuracy, further demonstrating its strengths and utility as a hybrid FL algorithm even accounting for the additional encryption and latency times which do not apply to its counterparts.

\begin{figure*}[h!]
\begin{center}
\includegraphics[width=\textwidth]{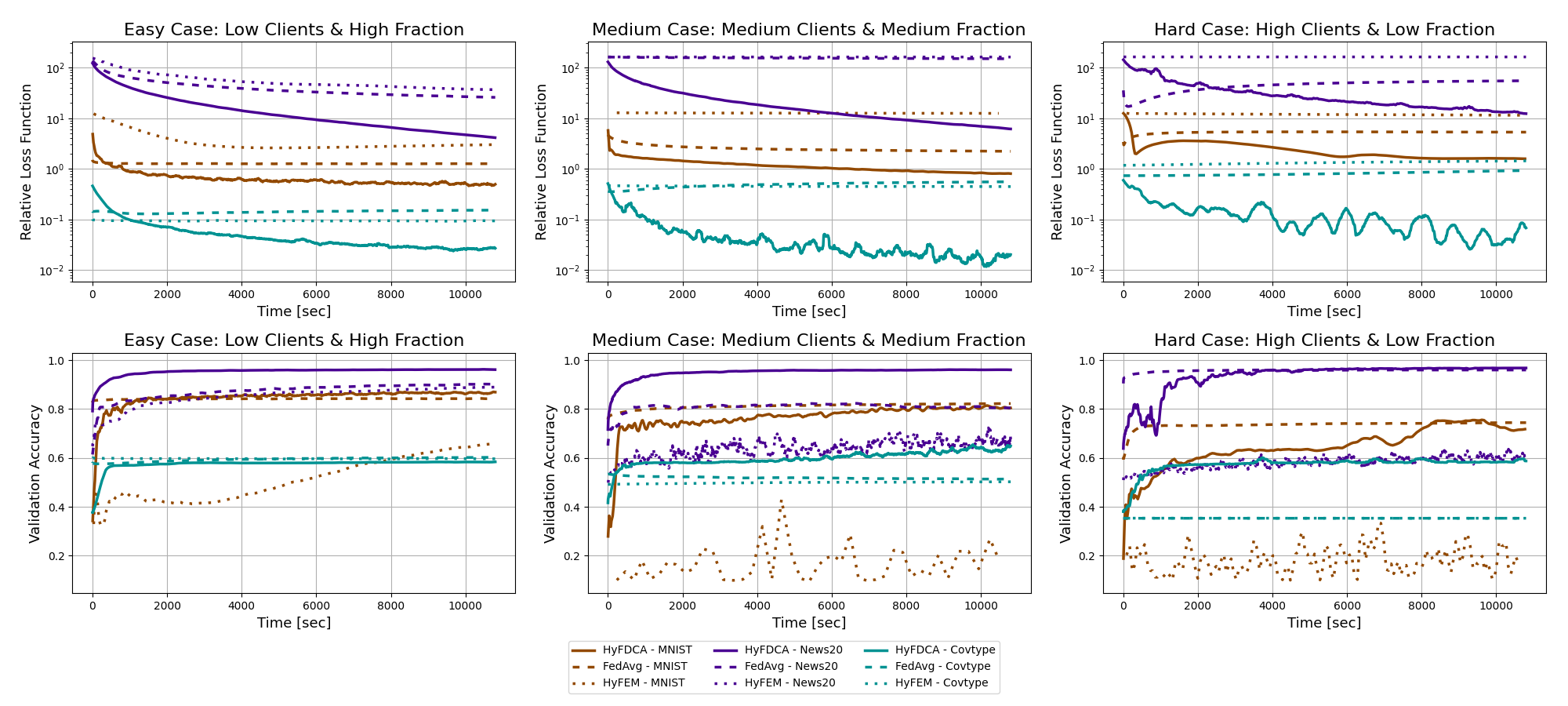}
\caption{Comparison of HyFDCA, FedAvg, and HyFEM over constant parallel-equivalent run-time in varying client-fraction settings.}
\label{equal_time_fig}
\end{center}
\end{figure*}

In addition, we provide complete results for HyFDCA for all 9 client-fraction settings examined to show the effect of different problem settings on the performance of across datasets. Namely, Figure \ref{figure2-mnist} shows both the relative objection function value and validation accuracy for MNIST, Figure \ref{figure2-news20} for News20, and Figure \ref{figure2-covtype} for Covtype.

\begin{figure*}[h!]
\begin{center}
\includegraphics[width=\textwidth]{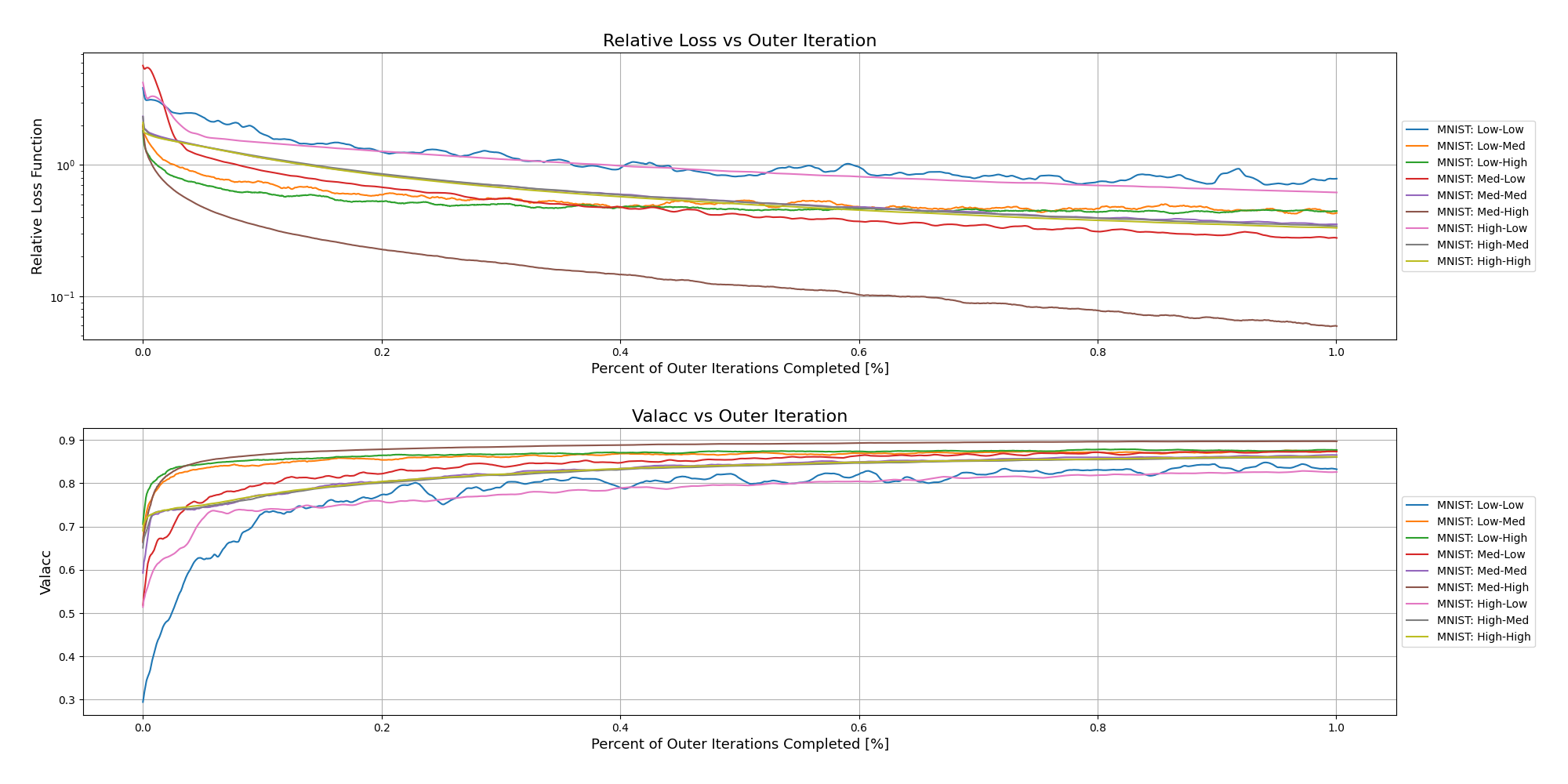}
\caption{Effect of number of clients and fraction of participating clients on HyFDCA performance on MNIST.}
\label{figure2-mnist}
\end{center}
\end{figure*}

\begin{figure*}[h!]
\begin{center}
\includegraphics[width=\textwidth]{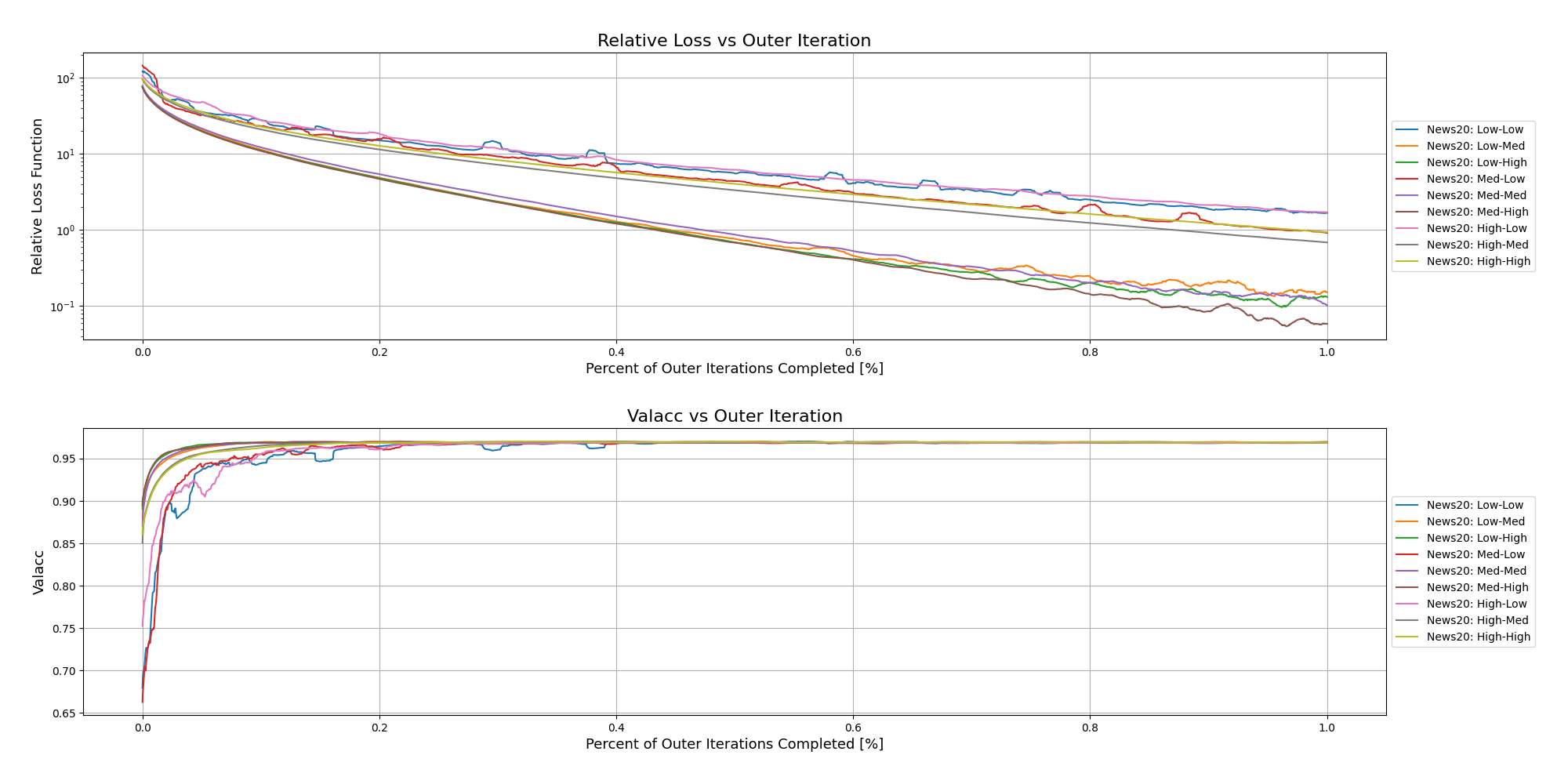}
\caption{Effect of number of clients and fraction of participating clients on HyFDCA performance on News20.}
\label{figure2-news20}
\end{center}
\end{figure*}

\begin{figure*}[h!]
\begin{center}
\includegraphics[width=\textwidth]{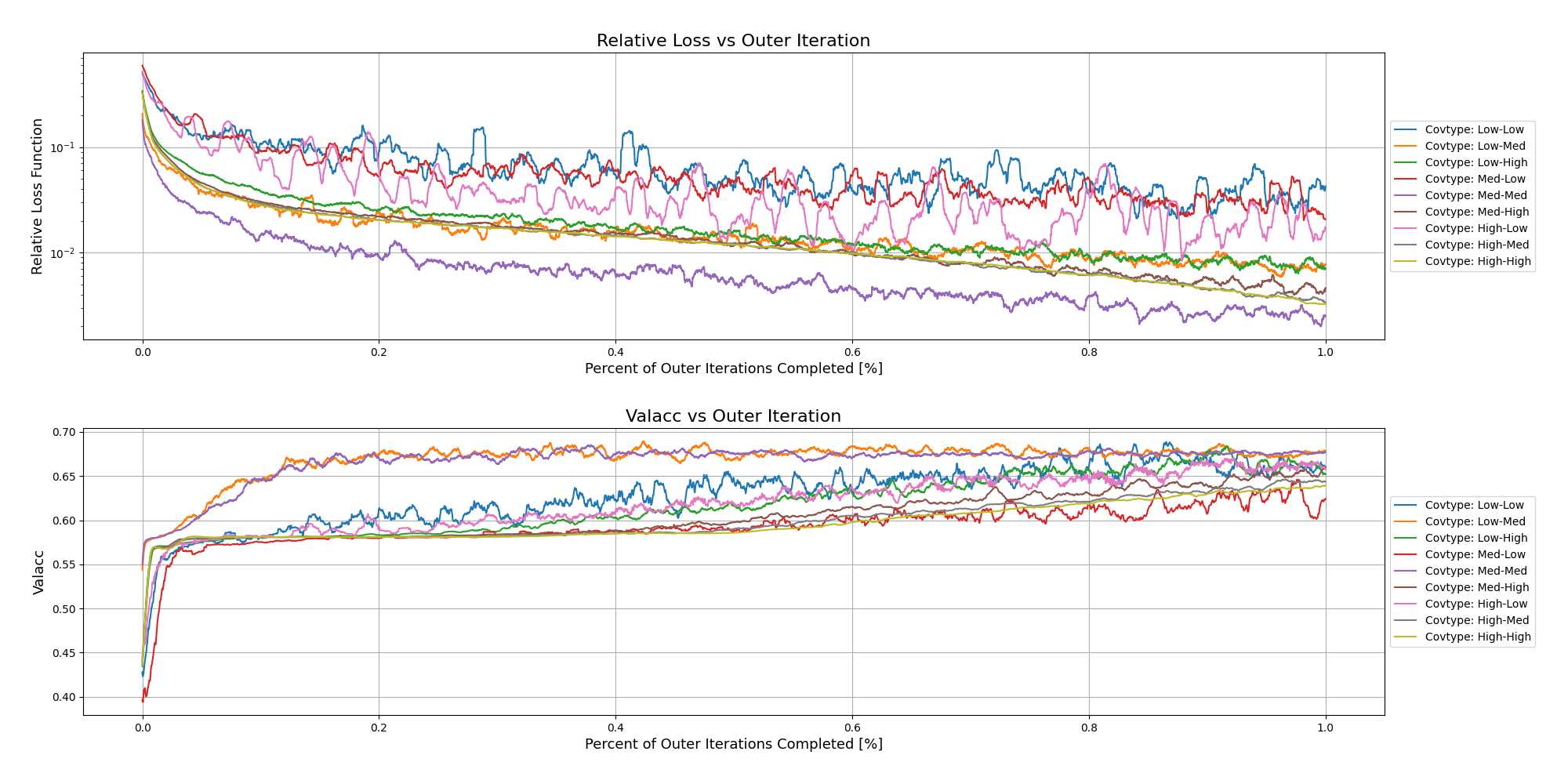}
\caption{Effect of number of clients and fraction of participating clients on HyFDCA performance on Covtype.}
\label{figure2-covtype}
\end{center}
\end{figure*}

Figure 4 of the main paper displays computational and encryption time data only for three client-fraction settings. The full data are represented in Tables \ref{mnist-times-table}, \ref{news20-times-table}, and \ref{covtype-times-table} for MNIST, News20 and Covtype, respectively.

\begin{table}[!htb]
\caption{Computational and Encryption Times for MNIST (seconds)}
\begin{center}
\begin{tabular}{r|SSSS}
 & \textbf{HyFDCA Computational} & \textbf{HyFDCA Encryption} & \textbf{FedAvg Computational} & \textbf{HyFEM Computational} \\
\hline
\textbf{5, 0.1} & 0.1736 & 0.1107 & 0.2268 & 1.210 \\
\textbf{5, 0.5} & 0.2472 & 3.164 & 0.2544 & 1.178 \\
\textbf{5, 0.9} & 0.2975 & 11.32 & 0.1543 & 1.177 \\
\textbf{500, 0.1} & 0.3560 & 7.827 & 0.1027 & 3.140 \\
\textbf{500, 0.5} & 1.207 & 13.67 & 0.1922 & 11.82 \\
\textbf{500, 0.9} & 1.827 & 431.0 & 0.2064 & 20.44 \\
\textbf{5000, 0.1} & 1.900 & 61.69 & 0.5775 & 22.46 \\
\textbf{5000, 0.5} & 7.604 & 1136. & 0.6021 & 108.1 \\
\textbf{5000, 0.9} & 13.62 & 3903. & 0.5455 & 194.8 \\

\label{mnist-times-table}
\end{tabular}
\end{center}
\end{table}

\begin{table}[h]
\caption{Computational and Encryption Times for News20 (seconds)}
\begin{center}
\begin{tabular}{r|SSSS}
& \textbf{HyFDCA Computational} & \textbf{HyFDCA Encryption} & \textbf{FedAvg Computational} & \textbf{HyFEM Computational} \\
\hline
\textbf{3, 3, 0.1} & 0.1402 & 0.3916 & 1.517 & 0.8817 \\
\textbf{3, 3, 0.5} & 0.1879 & 2.180 & 1.528 & 0.9693 \\
\textbf{3, 3, 0.9} & 0.2218 & 3.444 & 1.545 & 1.528 \\
\textbf{5, 5, 0.1} & 0.1793 & 1.095 & 1.452 & 1.598 \\
\textbf{5, 5, 0.5} & 0.2291 & 4.360 & 1.483 & 1.491 \\
\textbf{5, 5, 0.9} & 0.2749 & 7.671 & 1.485 & 2.003 \\
\textbf{12, 12, 0.1} & 0.2670 & 2.360 & 4.194 & 4.316 \\
\textbf{12, 12, 0.5} & 0.3625 & 10.81 & n/a & n/a \\
\textbf{12, 12, 0.9} & 0.4683 & 17.99 & n/a & n/a \\
\label{news20-times-table}
\end{tabular}
\end{center}
\end{table}

\begin{table}[h]
\caption{Computational and Encryption Times for Covtype (seconds)}
\begin{center}
\begin{tabular}{r|SSSS}
& \textbf{HyFDCA Computational} & \textbf{HyFDCA Encryption} & \textbf{FedAvg Computational} & \textbf{HyFEM Computational} \\
\hline
\textbf{3, 3, 0.1} & 0.04036 & 0.07758 & 0.04077 & 0.06974 \\
\textbf{3, 3, 0.5} & 0.0836 & 1.149 & 0.05855 & 0.07176 \\
\textbf{3, 3, 0.9} & 0.1374 & 0.4694 & 0.05638 & 0.04288 \\
\textbf{5, 5, 0.1} & 0.04416 & 0.09462 & 0.02457 & 0.1379 \\
\textbf{5, 5, 0.5} & 0.1113 & 2.920 & 0.02655 & 0.0008275 \\
\textbf{5, 5, 0.9} & 0.1936 & 1.001 & 0.03297 & 0.04010 \\
\textbf{12, 12, 0.1} & 0.1144 & 0.5975 & 0.001403 & 0.0008799 \\
\textbf{12, 12, 0.5} & 0.3472 & 2.771 & 0.002198 & 0.009640 \\
\textbf{12, 12, 0.9} & 0.5341 & 5.453 & 0.03100 & 0.007782 \\

\label{covtype-times-table}
\end{tabular}
\end{center}
\end{table}

\end{document}